\documentclass{article}
\usepackage[preprint]{neurips_2026}
\usepackage[utf8]{inputenc}
\usepackage[T1]{fontenc}
\usepackage{hyperref}
\usepackage{url}
\usepackage{booktabs}
\usepackage{array}
\usepackage{amsmath,amssymb,amsfonts,amsthm}
\usepackage{graphicx}
\usepackage{algorithm}
\usepackage{algorithmic}
\usepackage{nicefrac}
\usepackage{microtype}
\usepackage{courier}
\usepackage{xcolor}

\newtheorem{fact}{Fact}

\newcommand{\best}[1]{\mathbf{#1}}
\newcommand{\second}[1]{\underline{#1}}
\newcommand{\metriccell}[2]{\begin{tabular}[c]{@{}c@{}}{\normalsize $#1$}\\[-0.35ex]{\tiny$\pm\,#2$}\end{tabular}}
\newcolumntype{L}[1]{>{\raggedright\arraybackslash}p{#1}}

\title{An Information-Theoretic Evaluation Framework for Benchmark and Model Diagnosis in Knowledge Tracing}
\author{%
Houru Jiang$^{1,*}$, Zixi Wang$^{1,*}$, Tengteng Cheng$^2$, Xueyi Li$^2$\\
Mingliang Hou$^{2,\dagger}$, Jiaqi Zheng$^2$, Renqiang Luo$^3$, Teng Guo$^2$, Zitao Liu$^2$\\[0.5ex]
$^1$Jinan University, Guangzhou, China\\
$^2$Guangdong Institute of Smart Education, Jinan University, Guangzhou, China\\
$^3$Jilin University, Changchun, China\\
\texttt{teemohold@outlook.com}\\[0.3ex]
{\small $^*$Equal contribution. $^\dagger$Corresponding author.}
}
\hypersetup{hidelinks,
pdftitle={An Information-Theoretic Evaluation Framework for Benchmark and Model Diagnosis in Knowledge Tracing},
pdfauthor={Houru Jiang, Zixi Wang, Tengteng Cheng, Xueyi Li, Mingliang Hou, Jiaqi Zheng, Renqiang Luo, Teng Guo, Zitao Liu},
pdfsubject={Knowledge tracing; information-theoretic benchmark and model diagnosis},
pdfkeywords={knowledge tracing, context tree weighting, uncertainty, benchmark diagnosis}}

\begin{document}
\maketitle

\begin{abstract}
Knowledge tracing (KT) models are predominantly evaluated using aggregate metrics such as area under the curve (AUC) and accuracy. However, these global scores obscure where the remaining errors originate and fail to indicate whether a benchmark is approaching saturation. While estimating a global theoretical performance limit is challenging in realistic KT settings, it is possible to quantify local predictability. To address this, we propose an information-theoretic evaluation framework for KT benchmark diagnosis. We use Context Tree Weighting (CTW) on item-response histories and current-item queries as an operational causal uncertainty coordinate, while distinguishing it from the unobserved Local Irreducible Uncertainty (LIU) under the full KT information set. By projecting predictions onto this shared uncertainty coordinate, we evaluate model performance gains across distinct entropy bands rather than only at the global level. Comprehensive evaluations on NIPS Task 3/4 and Algebra 2005 reveal that model improvements are highly non-uniform. Modern KT models show substantial gains in high-entropy regions, and additional item-aware references, log-loss, and equal-frequency analyses support this localization. The framework also flags regions where apparent gains require checks for noise-sensitive behavior. By surfacing these local modeling failures alongside genuine gains, this approach provides a diagnostic tool for studying both residual predictive structure and the limitations of current KT benchmarks and models.

\end{abstract}

\section{Introduction}
Knowledge tracing (KT) predicts a student's future responses from their past interaction history, serving as a foundational task in artificial intelligence (AI) for education. As modern KT architectures grow increasingly sophisticated, performance on standard benchmarks continues to climb. However, the evaluation paradigm remains dominated by aggregate metrics such as area under the curve (AUC) and accuracy. These global scores are useful for leaderboard rankings, but they fail to answer two diagnostic questions that matter for progress: \emph{where does meaningful benchmark headroom remain?} and \emph{what kind of errors is a model reducing?} A marginal $+0.01$ boost in global AUC cannot distinguish between a model that refines already-easy cases, one that extracts genuine structure from difficult interactions, and one whose apparent gain is concentrated in unstable regions where outcomes are close to random under the observable protocol.

To address this gap, we propose an \emph{information-theoretic} evaluation framework for KT benchmark and model diagnosis. The framework separates two roles that aggregate evaluation conflates. First, \textbf{uncertainty characterization} establishes a shared local coordinate for prediction difficulty. Because the true Local Irreducible Uncertainty (LIU) is not directly observable in finite KT data, we use Context Tree Weighting (CTW) as an operational causal anchor: it assigns each prediction point a fold-consistent entropy estimate using past item-response history and the current queried item. Second, \textbf{predictive realization} measures how much task-facing KT models and anchor-side neural probes improve on that coordinate. This decoupling lets us ask not only whether a benchmark is moving on average, but also whether gains occur in low-uncertainty regions, high-uncertainty regions with exploitable structure, or high-uncertainty regions that call for item or data inspection.

Our contributions are as follows: (i) We introduce an information-theoretic evaluation paradigm that shifts KT assessment from global leaderboard scores to localized benchmark and model diagnosis; (ii) We formalize the separation between local irreducible uncertainty and realized predictive performance, with CTW serving as an operational causal anchor rather than as the ground-truth LIU; (iii) We provide a fold-consistent empirical protocol based on CTW-anchored entropy bands, gain-over-anchor curves, and item-level remaining-uncertainty/information-gain diagnostics; and (iv) Through broad aligned experiments on NIPS Task 3/4 and Algebra 2005, we show that the largest reliable gains occur in high-entropy regions while also identifying high-uncertainty, low-information-gain cases that warrant further reliability checks, with item-aware anchor and proper-score analyses testing the robustness of the localization.

\section{Related Work}
KT has progressed from Bayesian KT \citep{corbett1994knowledge} to neural sequence models such as DKT \citep{piech2015deep}, memory-based models such as DKVMN \citep{zhang2017dkvmn}, attention-based variants such as SAKT and AKT \citep{pandey2019sakt,ghosh2020akt}, and strong modern baselines such as Deep-IRT and simpleKT \citep{yeung2019deepirt,liu2022simplekt}. Toolkits and benchmarks including ASSISTments, EdNet, the NeurIPS 2020 Education Challenge, and pyKT have improved reproducibility \citep{heffernan2014assistments,choi2019ednet,wang2021neuripsedu,liu2022pykt}. However, these comparisons still primarily report aggregate metrics such as AUC and accuracy.

Our work is closer to probabilistic forecast evaluation. Information-theoretic uncertainty \citep{cover2006elements}, universal coding tools such as CTW \citep{willems1995ctw}, and proper scoring rules such as Brier score and log-loss \citep{brier1950verification,gneiting2007strictly} provide tools for separating uncertainty from realized predictive quality. We use CTW as a common causal anchor and evaluate model gains along entropy bands rather than only at the benchmark level. This makes the contribution diagnostic rather than architectural: the objective is to determine where benchmark progress is reliable, where model capacity changes predictions, and where apparent gains are more consistent with unstable or noise-like item behavior.

\begin{figure}[t]
\centering
\makebox[\linewidth][c]{\includegraphics[width=1\linewidth]{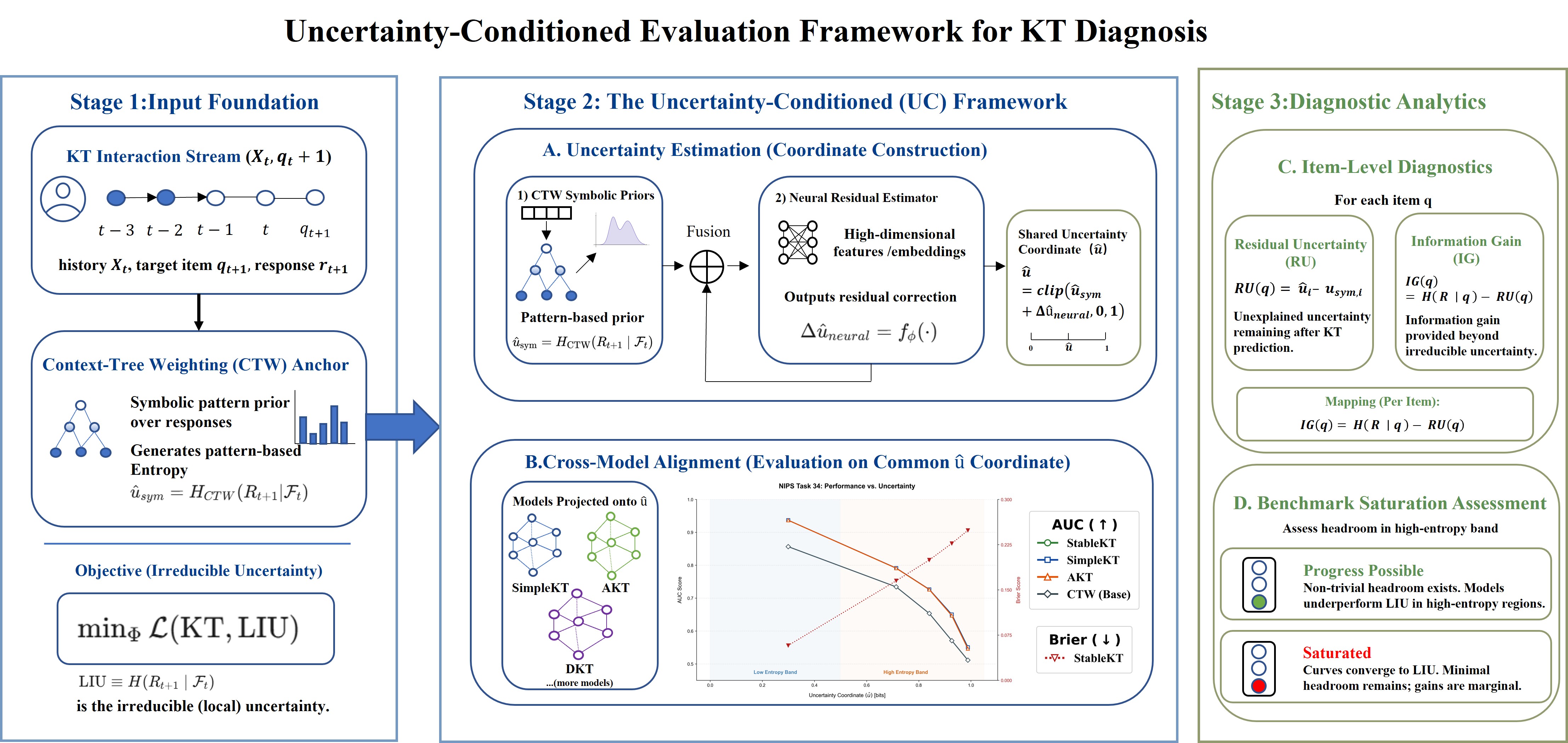}}
\caption{Overview of the proposed information-theoretic evaluation framework. A causal anchor assigns each prediction point to a shared uncertainty coordinate; model performance is then read along this coordinate to reveal where benchmark headroom remains and where model behavior warrants further diagnosis.}
\label{fig:framework-9page}
\end{figure}

\section{Framework}
\subsection{Problem setup and local uncertainty}
In standard KT, we observe a student's interaction history up to time $t$, denoted as $X_t=\{(q_1,r_1),\ldots,(q_t,r_t)\}$, where $q_k$ represents the practiced item and $r_k \in \{0,1\}$ represents the correctness of the response. Given the next item $q_{t+1}$, the task is to predict the probability of a correct response:
\begin{equation}
\hat{p}_{t+1} = \hat{P}(R_{t+1}=1 \mid X_t, q_{t+1}).    
\end{equation}

To formalize the information available at this prediction step, we define the observable filtration as $\mathcal{F}_t = \sigma(X_t, q_{t+1})$. Based on this, we define the \emph{local irreducible uncertainty} (LIU) of the prediction point as the conditional entropy of the true response $R_{t+1}$:
\begin{equation}
\mathrm{LIU}(X_t, q_{t+1}) = H\!\left(P(R_{t+1}\mid X_t,q_{t+1})\right).
\end{equation}
This term represents the theoretical entropy remaining under the admissible KT information boundary, i.e., the inherent unpredictability even if we had a perfect model for the student's historical data. In finite benchmark data, LIU is not directly observed; it must be studied through operational causal estimates that respect the information available at prediction time.

Rather than aggregating all predictions into a single global metric, our framework evaluates models locally. For each prediction instance $i$ across the dataset, we compute an operational CTW uncertainty coordinate $\hat{u}_{\mathrm{ctw},i}$ as a causal proxy for $\mathrm{LIU}_i$, partition the samples into distinct uncertainty buckets $\mathcal{B}_k$, and evaluate the model's performance $\mathrm{Perf}_m(k)$ within each specific bucket.

\subsection{Entropy-risk decomposition and causal anchor}
Let $P_t = P(R_{t+1} \mid \mathcal{F}_t)$ be the true conditional response distribution. For any $\mathcal{F}_t$-measurable probabilistic predictor $\hat{P}_\theta$, the expected log-loss (cross-entropy) naturally decomposes into two components:
\begin{fact}[Local entropy-risk decomposition]
\begin{equation}
\mathbb{E}\big[-\log \hat{P}_\theta(R_{t+1} \mid \mathcal{F}_t)\big]
=\underbrace{H(R_{t+1} \mid \mathcal{F}_t)}_{\text{mean LIU}}
+\underbrace{\mathbb{E}\big[\mathrm{KL}(P_t \,\|\, \hat{P}_\theta)\big]}_{\text{epistemic gap}}.
\end{equation}
\end{fact}
This decomposition illustrates that standard benchmark-level error inextricably mixes LIU with model mismatch (the epistemic gap). A global score can decrease either because a model reduces epistemic error in a meaningful region or because gains and failures cancel across regions with very different local uncertainty.

All entropies and KL divergences in the theoretical identities use base-2 logarithms, so they are measured in bits. The outer expectation in Fact~1 averages over prediction contexts; the local entropy is the entropy of $P_t$ at a given context.

To study uncertainty operationally, we use Context Tree Weighting (CTW) as an autoregressive causal anchor. In the real-data implementation, its context contains past item-response tokens, encoded as \texttt{I:item|R:resp}, and a query token \texttt{Q:item} for the current item. At prediction step $t+1$, these tokens are constructed from $X_t$ and $q_{t+1}$; the current response is observed only after prediction, and future responses are excluded. Thus item identity is a predictor input, while an explicit Q-matrix or skill-tag representation is not required. The corresponding uncertainty coordinate is $\hat{u}_{\mathrm{ctw},t}=H(p_{\mathrm{ctw}})$. This is a model-dependent estimate under the causal information boundary $\mathcal F_t=\sigma(X_t,q_{t+1})$, not a response-only coordinate or a certified Bayes limit. Gains over CTW can reflect limitations of its finite-depth context model even when predictors share this information boundary. We also test localization under an independent ItemRasch reference that models item difficulty and online ability differently (Appendix~\ref{app:robustness}).

Furthermore, we introduce a fused CTW+NN predictor. This model augments the pure CTW anchor by applying a neural refinement of the form $p_{\mathrm{ctw}} \mapsto \sigma(\mathrm{logit}(p_{\mathrm{ctw}}) + f_\phi(X_t, q_{t+1}))$. In our evaluation protocol, the pure CTW model defines the shared uncertainty axis (i.e., the buckets), while the CTW+NN model acts as an anchor-side neural probe. It shows whether exploitable structural signals still exist beyond the capabilities of the symbolic source model.

\subsection{Operational bands, remaining uncertainty, and information gain}
Given an aligned evaluation set, our empirical protocol proceeds as follows: First, we compute a fold-consistent CTW prediction $p_{\mathrm{ctw},i}$ for each evaluation point $i$. Second, we convert this probability into an estimated entropy score $\hat{u}_{\mathrm{ctw},i} = H(p_{\mathrm{ctw},i})$. Third, we partition the samples into distinct uncertainty buckets (using fixed-width bands in the original analysis and equal-frequency quintiles in the additional robustness analysis). Finally, we report both the absolute performance and the performance gain over the CTW anchor within each bucket. This resulting band view is deliberately model-agnostic: any KT baseline can be evaluated on this exact same CTW-defined coordinate.

The same autoregressive anchor also enables fine-grained, item-level diagnostics. For any given item $q$, let $H_q^{\mathrm{marg}}$ denote the entropy of its marginal empirical correctness, and let $H_q^{\mathrm{ctw}}$ denote its average CTW conditional entropy across all predictions for that item. We define the \emph{remaining uncertainty} as $\mathrm{RU}(q) = H_q^{\mathrm{ctw}}$ and the \emph{information gain} as $\mathrm{IG}(q) = H_q^{\mathrm{marg}} - H_q^{\mathrm{ctw}}$. A high $\mathrm{RU}$ identifies items whose outcomes remain unstable under the current observable protocol (e.g., potentially due to guessing, ambiguous wording, missing observables, or overly coarse step granularity); a high $\mathrm{IG}$ highlights items with lower anchor entropy than item-marginal entropy. This plug-in difference is a diagnostic rather than a guaranteed estimate of mutual information; anchor misspecification can affect either quantity.

\subsection{Benchmark and model diagnosis criteria}
The framework answers three diagnostic questions. First, \textbf{benchmark diagnosis} asks whether high-uncertainty regions with substantial support still admit reliable gains over the causal anchor; if strong models remain separated there, aggregate saturation claims are premature. Second, \textbf{model diagnosis} asks where a model improves and whether those improvements are supported by proper scoring behavior, not only by threshold accuracy. Third, \textbf{noise-sensitive behavior} asks whether apparent improvements occur in high-RU/low-IG regions where the item-conditioned anchor leaves substantial uncertainty. Such movements should not be interpreted as progress unless they are stable across folds and accompanied by proper-score improvements.

\section{Experimental Protocol}
We evaluate five-fold aligned exports on NIPS Task 3/4 and Algebra 2005. ASSIST2015 is reported as a global auxiliary check, while the local diagnostic figures focus on NIPS Task 3/4 and Algebra 2005 because these datasets support the band-level and item-level diagnostics used in the main figures. Each prediction point is assigned a CTW entropy value and placed into one of five entropy bands $[0.0,0.2),[0.2,0.4),[0.4,0.6),[0.6,0.8),[0.8,1.0]$. We report AUC, accuracy, and Brier score globally and within bands. Accuracy gives the decision-level view, while Brier score is used as a proper-score check on whether probability quality improves in the same region.

The model pool includes the CTW anchor, the CTW+NN neural probe, and a broad KT sweep spanning DKT, DKT+, DKT-forget, KQN, SAKT, SAINT, AKT, simpleKT, sparseKT, FoLiBiKT, stableKT, extraKT, Deep-IRT, ATKT, ATKTFix, DIMKT, ATDKT, and HawkesKT \citep{dktplus2018,dktforget2019,kqn2019,saint2020,atkt2021,sparsekt2024,folibikt2023,stablekt2024,extrakt2024,dimkt2022,atdkt2023,hawkes2021}. The main bucket coordinate is always computed from CTW alone, so task-facing KT baselines do not define the axis on which they are judged. Model diagnosis is then based on the alignment between gain-over-anchor curves, proper-score changes, fold consistency, and item-level RU/IG signals. The additional robustness analysis uses a single SimpleKT residual probe, denoted CTW+NN-SKT, on a separate common evaluation cohort; it is distinct from the original reported CTW+NN export. Appendix~\ref{app:robustness} reports its log-loss, calibration, item-aware reference, and quintile results separately.

\section{Results}
\subsection{Sanity check: same marginal correctness, different causal uncertainty}
Before real benchmarks, we verify the property needed from the anchor. We construct two binary response regimes with nearly identical marginal correctness rates. The first is an i.i.d. Bernoulli$(0.5)$ process. The second is a symmetric persistent two-state Markov process: the initial state is uniform over $\{0,1\}$, and each subsequent response copies the previous response with probability $0.95$ and flips with probability $0.05$. This process has stationary marginal correctness $0.5$ but forms long correct/incorrect runs. Across $64$ sequences of length $400$, the observed correctness rates are $0.492$ and $0.495$, yet average CTW entropy is $0.986$ bits for the i.i.d. regime and $0.318$ bits for the persistent regime. CTW predictive accuracy is $0.501$ versus $0.945$.

\begin{figure}[t]
\centering
\includegraphics[width=0.86\linewidth]{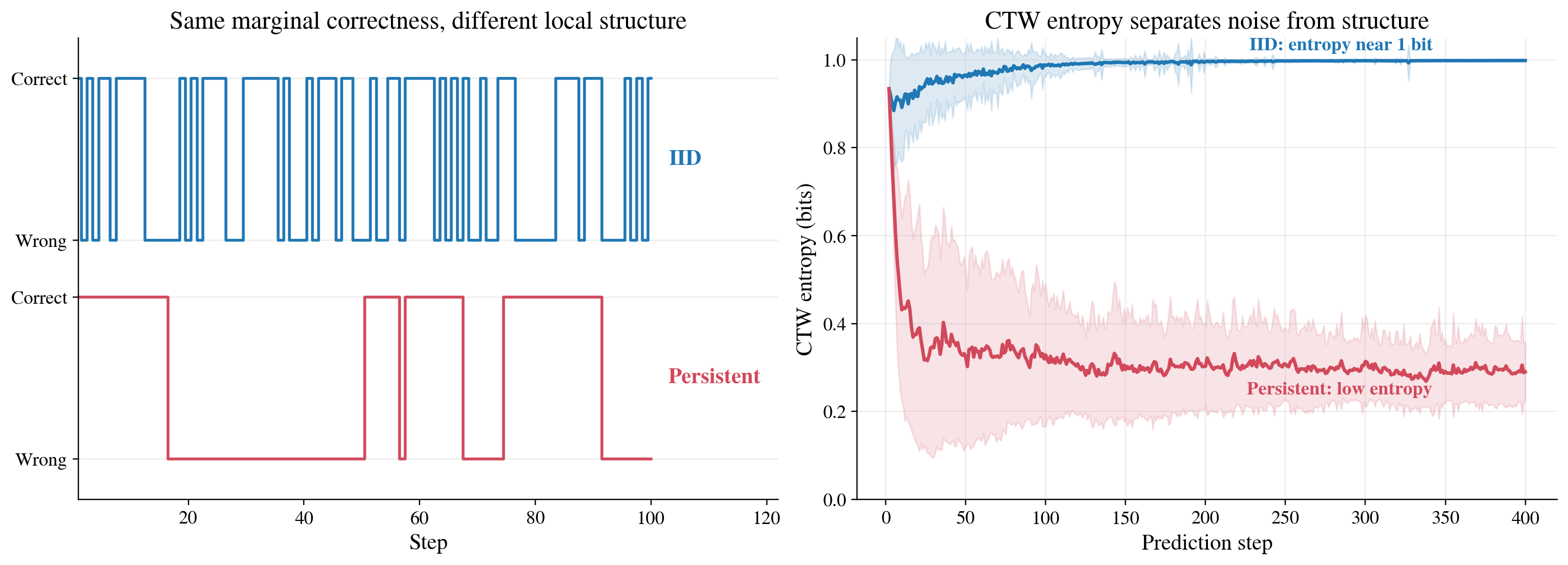}
\caption{Synthetic sanity check. Two regimes have nearly identical marginal correctness near $0.5$, but CTW separates noise-like responses from run-structured responses with strong causal predictability.}
\label{fig:same-acc-sanity-9page}
\end{figure}

\subsection{Global aligned performance is insufficient for diagnosis}
Global aligned performance is used as context rather than as the main diagnostic target. Across the aligned folds, CTW+NN improves the CTW anchor on all three datasets while remaining an anchor-side neural probe rather than the definition of the uncertainty axis: $+0.0823$ AUC and $+0.0690$ accuracy on NIPS Task 3/4, $+0.0835$ AUC and $+0.0160$ accuracy on Algebra 2005, and $+0.0455$ AUC and $+0.0142$ accuracy on ASSIST2015. The task-facing KT sweep is reported in Table~\ref{tab:overall-kt-9page}. These numbers show that stronger predictors improve on average, but they do not reveal whether the improvement comes from already predictable regions, genuinely difficult but structured regions, or unstable regions where apparent gains require caution. The central claim is local: global scores provide a sanity check, while the entropy-conditioned bands below test where benchmark headroom and model limitations actually lie.

\begin{table*}[t]
\caption{Task-facing KT baselines across the three benchmark datasets. Each metric cell is formatted as a two-line entry with the mean emphasized and the standard deviation reduced for readability. Entries are mean $\pm$ standard deviation over five folds; lower Brier is better. Best entries are bolded and second-best entries are underlined within each dataset/metric block. The unified table keeps only the KT models available on all three datasets. All NIPS Task 3/4 rows use the shared aligned subset of $1{,}354{,}735$ points. For Algebra 2005, all rows use $464{,}725$ points except ATKT ($271{,}122$) and ATKTFix ($278{,}835$). For ASSIST2015 under the convention $concept=question$, all rows use $672{,}960$ points.}
\label{tab:overall-kt-9page}
\centering
\scriptsize
\setlength{\tabcolsep}{3pt}
\begin{tabular}{lccccccccc}
\toprule
& \multicolumn{3}{c}{NIPS Task 3/4} & \multicolumn{3}{c}{Algebra 2005} & \multicolumn{3}{c}{ASSIST2015} \\
\cmidrule(lr){2-4} \cmidrule(lr){5-7} \cmidrule(lr){8-10}
Model & AUC & ACC & Brier & AUC & ACC & Brier & AUC & ACC & Brier \\
\midrule
simpleKT & \metriccell{\best{0.8014}}{0.0003} & \metriccell{\second{0.7309}}{0.0004} & \metriccell{\second{0.1798}}{0.0005} & \metriccell{0.8147}{0.0225} & \metriccell{0.8065}{0.0072} & \metriccell{0.1389}{0.0060} & \metriccell{0.7231}{0.0012} & \metriccell{0.7502}{0.0006} & \metriccell{0.1713}{0.0003} \\
stableKT & \metriccell{\second{0.8011}}{0.0005} & \metriccell{\best{0.7310}}{0.0003} & \metriccell{\best{0.1796}}{0.0003} & \metriccell{0.8200}{0.0181} & \metriccell{0.8070}{0.0063} & \metriccell{0.1365}{0.0052} & \metriccell{0.7256}{0.0010} & \metriccell{0.7506}{0.0002} & \metriccell{0.1707}{0.0002} \\
AKT & \metriccell{0.7999}{0.0005} & \metriccell{0.7292}{0.0008} & \metriccell{0.1805}{0.0004} & \metriccell{\second{0.8277}}{0.0015} & \metriccell{\second{0.8110}}{0.0015} & \metriccell{\second{0.1334}}{0.0007} & \metriccell{0.7262}{0.0008} & \metriccell{0.7519}{0.0008} & \metriccell{0.1708}{0.0002} \\
FoLiBiKT & \metriccell{0.7993}{0.0008} & \metriccell{0.7291}{0.0010} & \metriccell{0.1809}{0.0002} & \metriccell{0.8026}{0.0141} & \metriccell{0.8017}{0.0049} & \metriccell{0.1413}{0.0043} & \metriccell{0.7261}{0.0013} & \metriccell{\second{0.7520}}{0.0004} & \metriccell{0.1707}{0.0004} \\
extraKT & \metriccell{0.7982}{0.0020} & \metriccell{0.7285}{0.0016} & \metriccell{0.1810}{0.0011} & \metriccell{0.8215}{0.0029} & \metriccell{0.8064}{0.0018} & \metriccell{0.1365}{0.0011} & \metriccell{0.7257}{0.0004} & \metriccell{0.7506}{0.0002} & \metriccell{0.1712}{0.0001} \\
sparseKT-soft & \metriccell{0.7979}{0.0002} & \metriccell{0.7276}{0.0003} & \metriccell{0.1813}{0.0002} & \metriccell{0.8155}{0.0009} & \metriccell{0.8031}{0.0011} & \metriccell{0.1395}{0.0006} & \metriccell{0.7206}{0.0013} & \metriccell{0.7494}{0.0004} & \metriccell{0.1718}{0.0003} \\
sparseKT-topk & \metriccell{0.7974}{0.0010} & \metriccell{0.7271}{0.0008} & \metriccell{0.1814}{0.0004} & \metriccell{0.8137}{0.0012} & \metriccell{0.8020}{0.0009} & \metriccell{0.1407}{0.0002} & \metriccell{0.7161}{0.0066} & \metriccell{0.7490}{0.0024} & \metriccell{0.1728}{0.0016} \\
ATKT & \metriccell{0.7776}{0.0062} & \metriccell{0.7086}{0.0045} & \metriccell{0.1904}{0.0026} & \metriccell{\best{0.8373}}{0.0229} & \metriccell{\best{0.8179}}{0.0189} & \metriccell{\best{0.1280}}{0.0128} & \metriccell{\best{0.8069}}{0.0100} & \metriccell{\best{0.7787}}{0.0062} & \metriccell{\best{0.1494}}{0.0035} \\
DKT+ & \metriccell{0.7660}{0.0006} & \metriccell{0.7010}{0.0007} & \metriccell{0.1947}{0.0003} & \metriccell{0.8098}{0.0016} & \metriccell{0.8078}{0.0006} & \metriccell{0.1372}{0.0005} & \metriccell{\second{0.7283}}{0.0007} & \metriccell{0.7510}{0.0004} & \metriccell{\second{0.1702}}{0.0002} \\
DKT & \metriccell{0.7659}{0.0002} & \metriccell{0.7005}{0.0004} & \metriccell{0.1948}{0.0001} & \metriccell{0.8091}{0.0007} & \metriccell{0.8072}{0.0005} & \metriccell{0.1374}{0.0004} & \metriccell{0.7267}{0.0005} & \metriccell{0.7506}{0.0006} & \metriccell{0.1706}{0.0001} \\
KQN & \metriccell{0.7626}{0.0011} & \metriccell{0.6982}{0.0006} & \metriccell{0.1961}{0.0005} & \metriccell{0.7883}{0.0062} & \metriccell{0.7985}{0.0020} & \metriccell{0.1437}{0.0017} & \metriccell{0.7250}{0.0007} & \metriccell{0.7500}{0.0003} & \metriccell{0.1709}{0.0001} \\
SAINT & \metriccell{0.7621}{0.0303} & \metriccell{0.6964}{0.0258} & \metriccell{0.1975}{0.0105} & \metriccell{0.7660}{0.0133} & \metriccell{0.7780}{0.0031} & \metriccell{0.1607}{0.0039} & \metriccell{0.7007}{0.0020} & \metriccell{0.7441}{0.0015} & \metriccell{0.1765}{0.0009} \\
Deep-IRT & \metriccell{0.7599}{0.0029} & \metriccell{0.6951}{0.0025} & \metriccell{0.1971}{0.0013} & \metriccell{0.7949}{0.0020} & \metriccell{0.8006}{0.0007} & \metriccell{0.1421}{0.0005} & \metriccell{0.7185}{0.0003} & \metriccell{0.7501}{0.0004} & \metriccell{0.1719}{0.0001} \\
ATKTFix & \metriccell{0.7599}{0.0032} & \metriccell{0.6956}{0.0024} & \metriccell{0.1973}{0.0012} & \metriccell{0.7841}{0.0051} & \metriccell{0.7941}{0.0035} & \metriccell{0.1453}{0.0017} & \metriccell{0.7213}{0.0035} & \metriccell{0.7485}{0.0013} & \metriccell{0.1719}{0.0008} \\
SAKT & \metriccell{0.7481}{0.0004} & \metriccell{0.6844}{0.0010} & \metriccell{0.2031}{0.0011} & \metriccell{0.7867}{0.0037} & \metriccell{0.7964}{0.0016} & \metriccell{0.1477}{0.0022} & \metriccell{0.7112}{0.0004} & \metriccell{0.7472}{0.0001} & \metriccell{0.1756}{0.0007} \\
\bottomrule
\end{tabular}
\end{table*}

\subsection{Benchmark diagnosis through entropy bands}
The main empirical pattern is local rather than global. On NIPS Task 3/4, the accuracy improvement of CTW+NN over CTW grows from $+0.0026$ in the lowest entropy band to $+0.0886$ in the highest. On Algebra 2005, gains are small in lower bands but reach $+0.0283$ accuracy and $0.0681$ Brier reduction in the highest band. Moreover, the highest band is not a tiny tail: it contains roughly $74.8\%$ of NIPS Task 3/4 samples and $55.6\%$ of Algebra 2005 samples in the shared aligned subsets. These high-entropy regions are therefore central to the benchmark rather than peripheral outliers.

\begin{figure*}[t]
\centering
\includegraphics[width=0.98\textwidth]{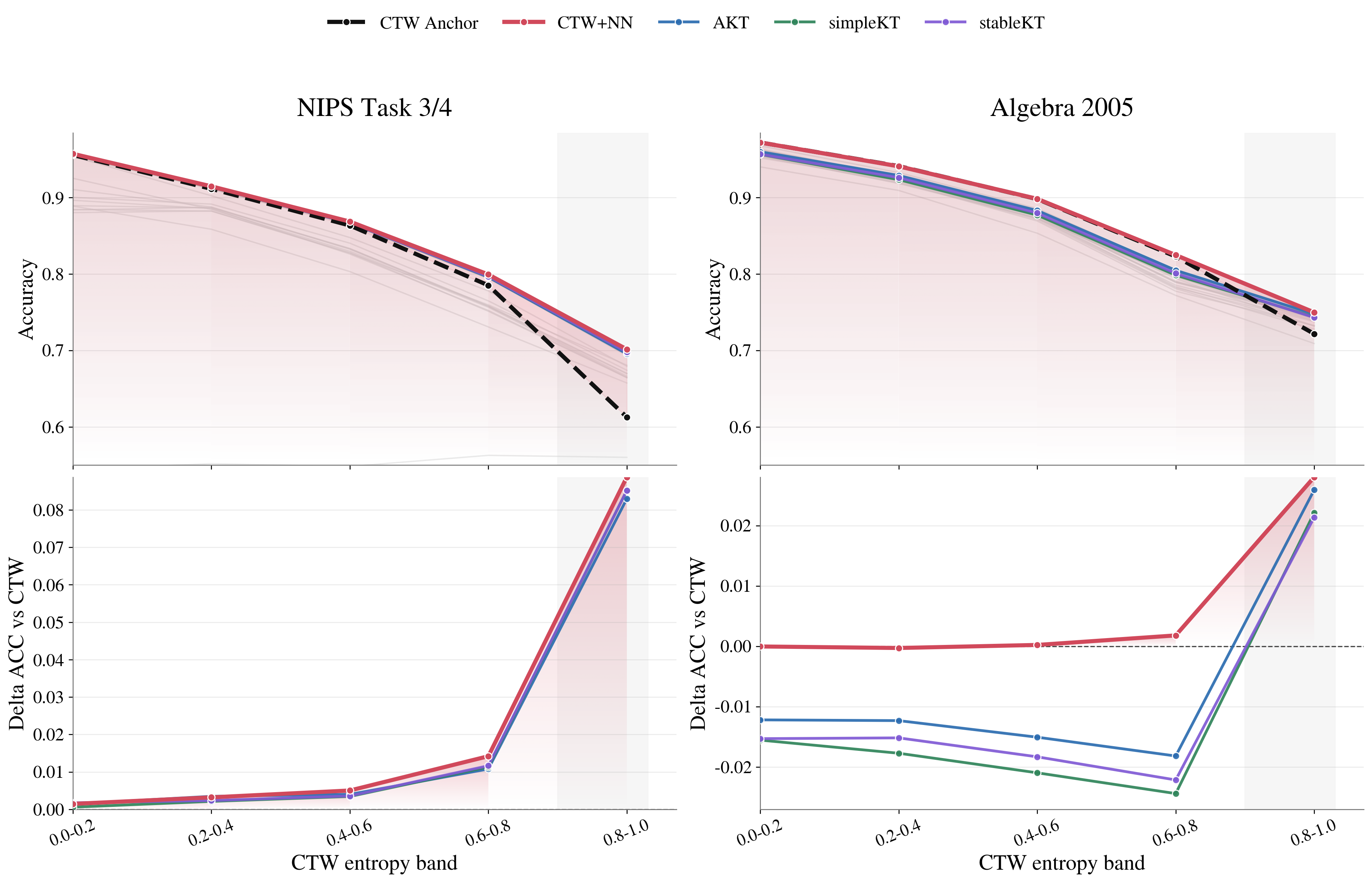}
\caption{Entropy-conditioned accuracy bands under the CTW anchor. Top: absolute accuracy by CTW entropy band. Bottom: accuracy gain relative to CTW. Higher CTW-entropy bands have lower accuracy, but the largest gains over CTW also occur in these highest-entropy regions, indicating difficult but not exhausted benchmark regions.}
\label{fig:band-accuracy-9page}
\end{figure*}

\begin{table}[t]
\caption{CTW+NN minus CTW improvement by CTW entropy band. Positive $\Delta$ACC and negative $\Delta$Brier indicate gains over the anchor.}
\label{tab:band-gains-9page}
\centering
\small
\begin{tabular}{lrrrr}
\toprule
\setlength{\tabcolsep}{3pt}
Band & \multicolumn{2}{c}{NIPS Task 3/4} & \multicolumn{2}{c}{Algebra 2005} \\
\cmidrule(lr){2-3}\cmidrule(lr){4-5}
 & $\Delta$ACC & $\Delta$Brier & $\Delta$ACC & $\Delta$Brier \\
\midrule
$[0.0,0.2)$ & $+0.0026$ & $-0.0075$ & $-0.0002$ & $-0.0006$ \\
$[0.2,0.4)$ & $+0.0037$ & $-0.0144$ & $+0.0002$ & $-0.0019$ \\
$[0.4,0.6)$ & $+0.0055$ & $-0.0161$ & $-0.0001$ & $-0.0041$ \\
$[0.6,0.8)$ & $+0.0147$ & $-0.0229$ & $+0.0021$ & $-0.0087$ \\
$[0.8,1.0]$ & $\mathbf{+0.0886}$ & $\mathbf{-0.0378}$ & $\mathbf{+0.0283}$ & $\mathbf{-0.0681}$ \\
\bottomrule
\end{tabular}
\end{table}

\subsection{Model diagnosis through gain profiles}
The band profile should be read as a diagnostic object rather than as a replacement for global model ranking. The absolute-accuracy row in Figure~\ref{fig:band-accuracy-9page} answers a first-order question: does the CTW coordinate actually stratify prediction difficulty? In both main datasets, the answer is yes. Accuracy falls as CTW entropy increases, and this trend is visible not only for the CTW anchor but also for the KT baselines. This matters because the coordinate is computed without access to future labels or to the task-facing KT model being evaluated. It therefore provides a common, causal axis on which different predictors can be compared.

The delta row answers a second-order question: where is there still exploitable headroom relative to the anchor? If a benchmark were nearly exhausted under the current observable protocol, then high-entropy regions should show little remaining separation among strong models. That is not the observed pattern. On NIPS Task 3/4, the CTW+NN neural probe barely moves the lowest-entropy band but gains almost nine accuracy points in the highest band. The main KT baselines follow the same qualitative profile. This shows that the difficult region is not merely random noise: it is difficult for the symbolic anchor, but still contains structure that stronger predictors can exploit.

Algebra 2005 shows a weaker low-band pattern but the same high-band conclusion. In the first four entropy bands, CTW+NN is close to CTW in accuracy, while several KT baselines can sit below the anchor depending on coverage and model family. In the highest band, however, CTW+NN separates clearly, and the Brier reduction is largest. This is exactly the regime where a local diagnostic is more informative than a global score: aggregate performance averages together easy regions that are already largely solved and high-uncertainty regions where model choice still matters.

This interpretation also clarifies why CTW+NN is included. It is not used to define the benchmark axis, and it is not presented as a new KT leaderboard method. Its role is to act as an anchor-side neural probe: if a simple neural refinement of the causal anchor improves most in high-entropy regions, then those regions should not be dismissed as irreducible noise. They are hard, but not exhausted. By contrast, flat delta curves in high-entropy bands across strong models would indicate a more saturated benchmark under the same observable protocol.

The same profile is also a model-diagnosis tool, but the roles must remain separate: CTW+NN is an anchor-side probe for remaining signal, not a task-facing KT baseline. For task-facing models, a gain is most credible when threshold accuracy and Brier score move in the same direction within the same region. Isolated accuracy changes in high-entropy regions should be inspected rather than immediately counted as progress, because high entropy can arise from both exploitable hidden structure and effectively random variation under the observed features. This is why the framework pairs band-level gains with item-level RU/IG diagnostics: the band profile identifies where a model changes behavior, while RU/IG asks whether that region contains contextual signal or unstable item behavior.

\subsection{Sensitivity and noise-sensitive item diagnostics}
Two additional checks support the same interpretation. First, repartitioning samples by a refined uncertainty coordinate $\hat u_{refined}=H(\hat p_{fused})$ preserves the non-uniform headroom profile, while smoothing the NIPS gain curve because the fused predictor already absorbs part of the structure left unresolved by CTW. On Algebra 2005, higher refined-uncertainty buckets still exhibit larger accuracy gains over the CTW anchor; on NIPS Task 3/4, the peak moves from the last bucket to the upper-middle region, consistent with the refined coordinate already encoding part of the additional signal.

\begin{figure}[t]
\centering
\includegraphics[width=0.88\linewidth]{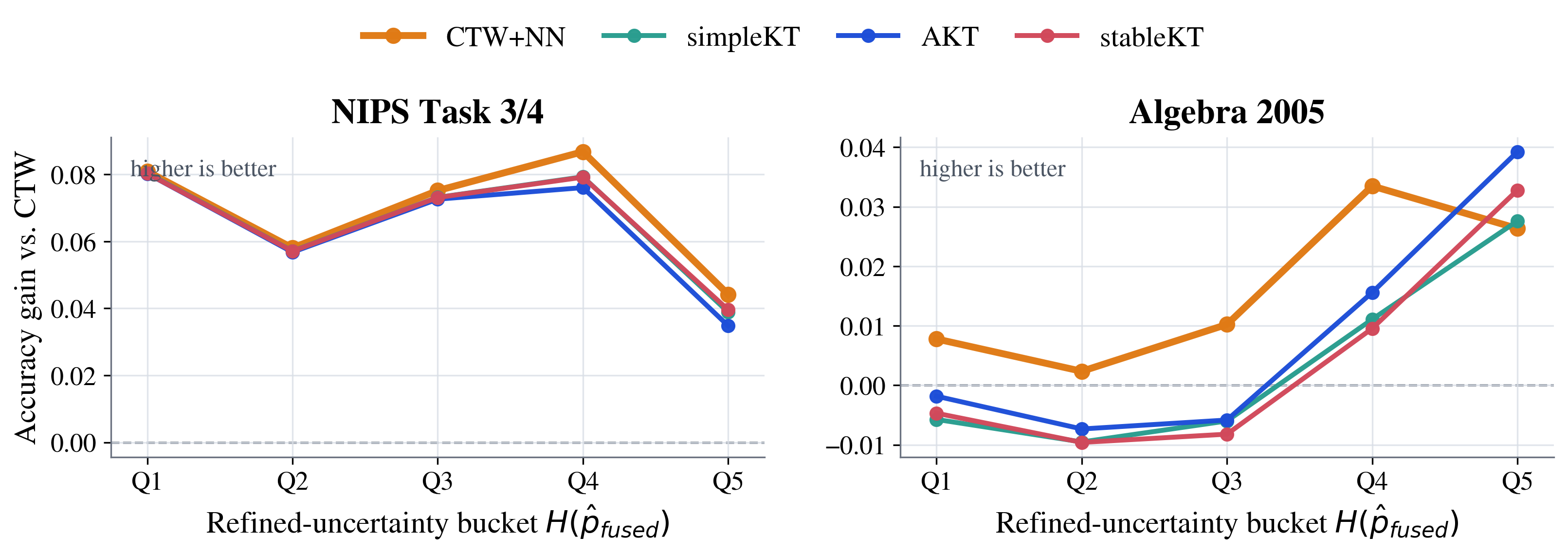}
\caption{Sensitivity analysis using the refined uncertainty coordinate $\hat u_{refined}=H(\hat p_{fused})$. The compact view focuses on accuracy gains relative to the CTW anchor, the primary headroom signal used in the band analysis. The refined coordinate preserves a non-uniform gain profile while avoiding the visual clutter of duplicating absolute-metric curves.}
\label{fig:refined-9page}
\end{figure}

Second, item-level RU and IG separate item regions whose outcomes remain unstable from regions where history carries strong predictive signal. This is the main mechanism by which the framework surfaces noise-sensitive model behavior. High-RU/low-IG items motivate inspection for ambiguity, guessing, missing observables, or problematic granularity; if a model's apparent improvement is concentrated in such items without proper-score or fold-stability support, the improvement is a warning signal for unstable prediction behavior; it does not by itself establish that the model has learned random noise. High-IG items, in contrast, point to regions where better temporal or prerequisite modeling should pay off. Figure~\ref{fig:item-9page} shows this support-filtered diagnostic map.

\begin{figure}[t]
\centering
\includegraphics[width=0.82\linewidth]{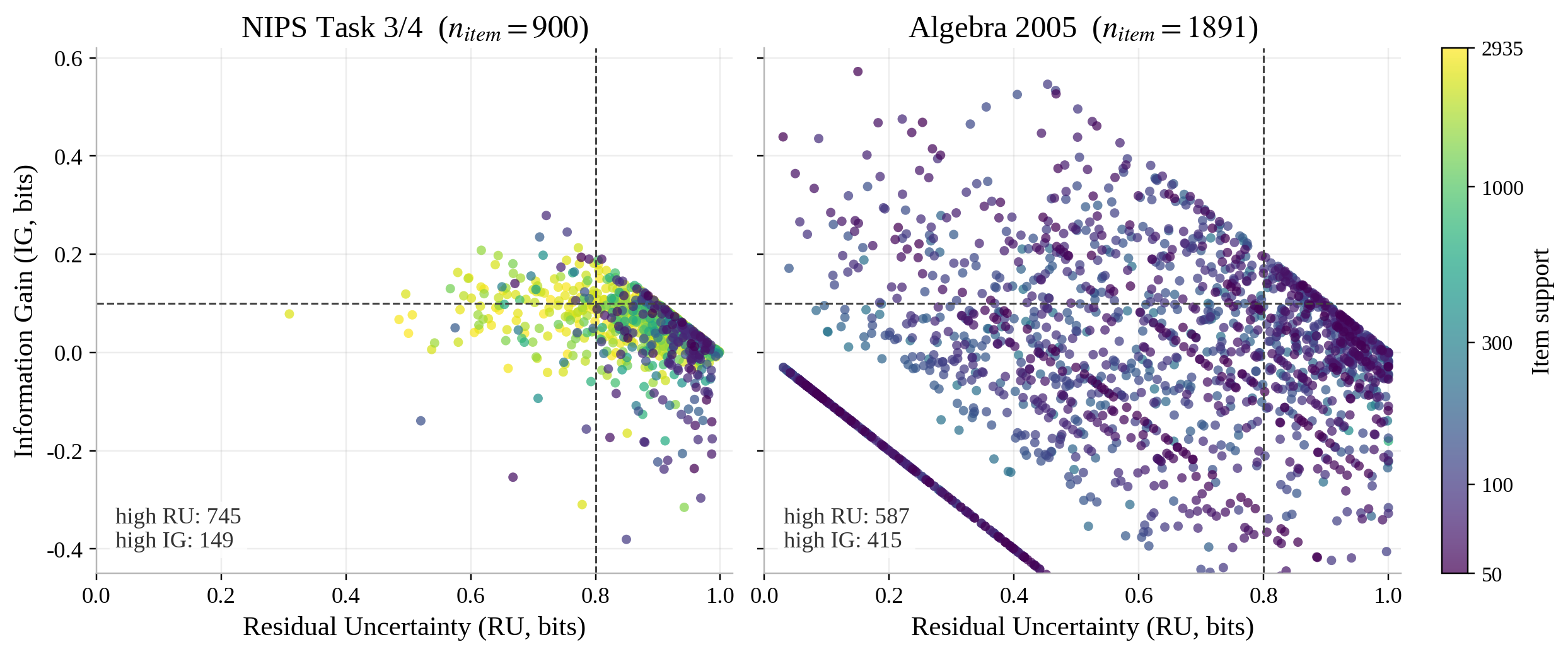}
\caption{Item-level diagnostic map under the CTW anchor. RU captures remaining item uncertainty and IG captures contextual signal beyond item marginal correctness. High-RU/low-IG regions are the main warning cases for noise-sensitive model behavior, while high-IG regions indicate contextual structure that models can exploit.}
\label{fig:item-9page}
\end{figure}

\begin{figure}[t]
\centering
\includegraphics[width=\linewidth]{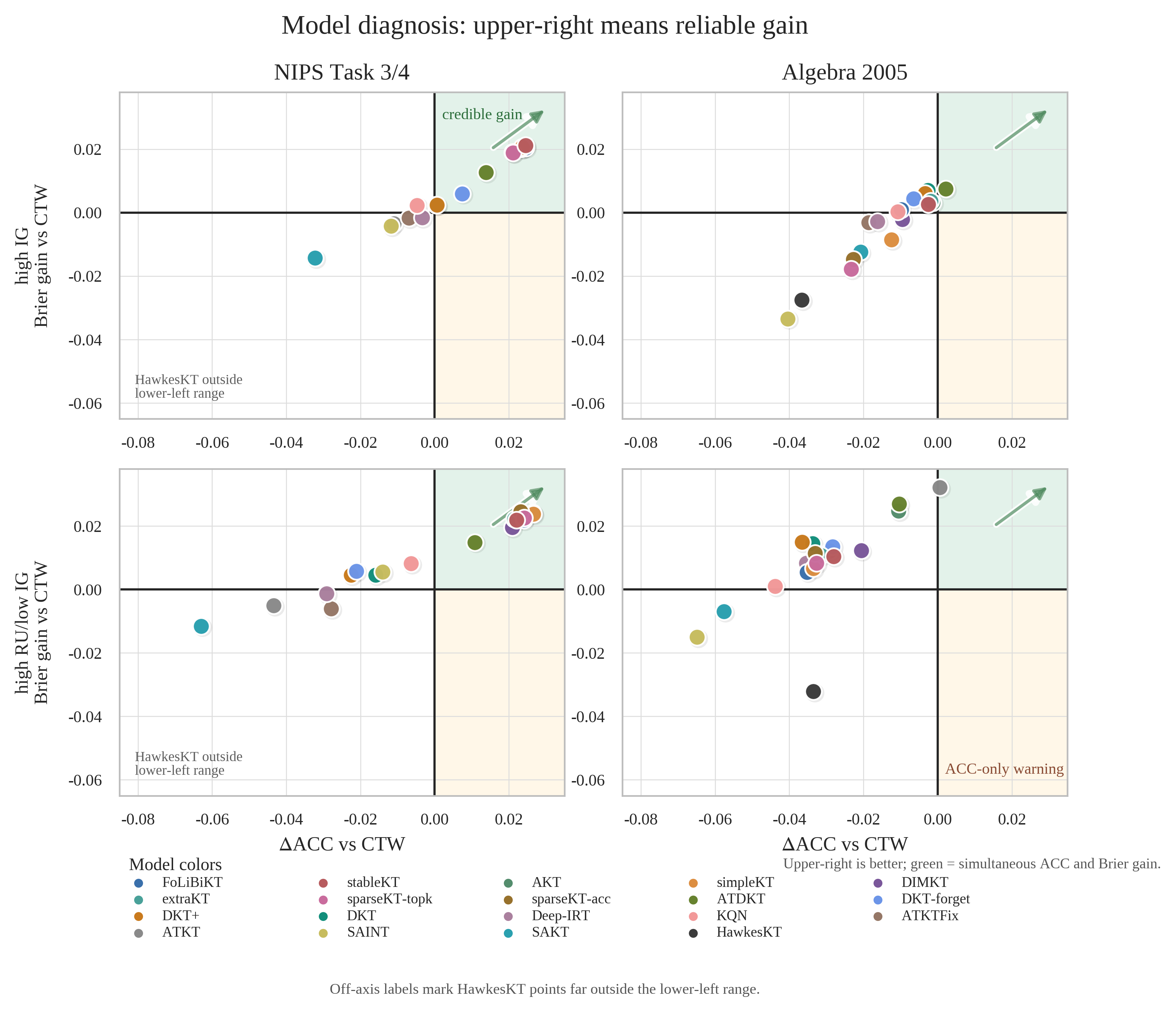}
\caption{Model diagnosis scatter plot for task-facing KT baselines. Columns separate datasets and rows separate RU/IG regions. Each point is one KT model, with color identifying the model. The horizontal axis reports $\Delta$ACC versus CTW and the vertical axis reports Brier gain ($-\Delta$Brier) versus CTW, so farther to the upper-right is better. Points in the green upper-right quadrant improve both threshold accuracy and probability quality; points moving right but not up indicate accuracy-only gains that require caution.}
\label{fig:model-diagnosis-ruig}
\end{figure}

The RU/IG view makes the distinction concrete. High-RU/low-IG regions have high anchor entropy relative to item-marginal entropy, so apparent model gains on them require proper-score and fold-stability checks before being treated as progress. Figure~\ref{fig:model-diagnosis-ruig} compares all task-facing baselines with a direct two-score diagnostic: the $x$-coordinate asks whether accuracy improves over CTW, the $y$-coordinate asks whether the proper score improves over CTW, and the green quadrant marks gains supported by both criteria. The two-score view converts the RU/IG map from an item-level warning signal into a model-level reliability check. The strongest NIPS models cluster in the green quadrant for high-IG regions, whereas Algebra and high-RU/low-IG panels contain more leftward or low-$y$ points. Thus uncertainty bands identify \emph{where} headroom remains, while RU/IG distinguishes exploitable structure from unstable item behavior.

\paragraph{Robustness to the reference and evaluation resolution.}
An independent ItemRasch reference uses an empirical-Bayes item estimate and a causal online student state, providing an alternative to CTW's item-token context model. In its highest-entropy band, CTW+NN-SKT improves accuracy by $+0.0420$ on NIPS and $+0.0788$ on Algebra, with favorable Brier and log-loss changes and five-fold confidence intervals. On the CTW coordinate, highest-band log-loss changes are $-0.0902$ and $-0.1578$ nats, respectively. Equal-frequency quintiles resolve gradients inside the large fixed-width band, and more than $93\%$ of high-CTW-entropy points occur after interaction 20. These checks support localization across alternative item-conditioned reference models and beyond cold start; full results and protocol distinctions appear in Appendix~\ref{app:robustness}.

\section{Discussion and Limitations}
The results support information-theoretic reporting as a benchmark- and model-diagnosis tool. A benchmark close to saturation should show tight clustering of strong models especially in high-uncertainty regions. We do not observe that pattern for NIPS Task 3/4 or Algebra 2005: the hardest CTW-defined regions still contain the largest gains. This suggests that aggregate KT scores hide substantial local structure and that residual structure remains under the tested references and evaluation protocols; it does not quantify distance to the Bayes limit.

CTW is used here as a causal anchor and common analysis coordinate, not as a complete model of student learning. This localizes where headroom remains while avoiding the confound that would arise if every model defined its own uncertainty axis. The refined-coordinate and Brier checks reduce the risk that the conclusion is an artifact of CTW-only bucketization or threshold accuracy. At the same time, the framework avoids treating every high-entropy gain as progress: high-RU/low-IG items expose regions where model changes should be interpreted as noise-sensitive unless supported by proper-score improvement and fold stability.

The analysis has several scope boundaries. The operational coordinates are model-dependent estimates, and their entropy is not a certified lower bound on achievable error. Five-fold intervals describe variation across the retained folds; overlapping training sets limit interpretations based on independent-fold assumptions. The additional CTW+NN-SKT analysis uses a distinct common cohort and must not be numerically pooled with the original reported branch. First, alternative uncertainty coordinates may reveal complementary structure. Second, some raw KT exports have dataset-specific coverage differences, so the main claims emphasize shared aligned subsets. Third, Figure~\ref{fig:model-diagnosis-ruig} reports model diagnosis at the RU/IG-region level, while the paper reports item-level RU/IG summaries for NIPS Task 3/4 and Algebra 2005.

\section{Conclusion}
We introduced an information-theoretic evaluation framework for KT benchmark and model diagnosis. By separating uncertainty characterization from predictive realization and using CTW entropy as a shared causal coordinate, the framework shows where model progress occurs. Across broad aligned sweeps, the largest gains concentrate in high-uncertainty regions, demonstrating that these benchmarks contain substantial exploitable structure beyond what aggregate scores reveal.

The practical message is that KT benchmarks should not be judged only by aggregate score movement. A small global improvement can be meaningful if it occurs in a difficult, high-uncertainty region with substantial support and proper-score improvement; conversely, gains concentrated in high-RU/low-IG regions may reflect noise-sensitive behavior. A minimal adoption path is straightforward: report global metrics, add entropy-conditioned gain curves on a fixed causal coordinate, and inspect RU/IG regions when the curve shows remaining headroom.

\begin{ack}
This work received no funding. The authors declare no competing interests.
\end{ack}

\clearpage
\bibliographystyle{plainnat}
\bibliography{refs}

@article{corbett1994knowledge,
	title = {Knowledge tracing: {Modeling} the acquisition of procedural knowledge},
	volume = {4},
	issn = {1573-1391},

url = {https://doi.org/10.1007/BF01099821},
	doi = {10.1007/BF01099821},

language = {en},
	number = {4},

journal = {User Modeling and User-Adapted Interaction},
	author = {Corbett, Albert T. and Anderson, John R.},
	month = dec,
	year = {1994},

pages = {253--278},
}

@inproceedings{piech2015deep,
 author = {Piech, Chris and Bassen, Jonathan and Huang, Jonathan and Ganguli, Surya and Sahami, Mehran and Guibas, Leonidas and Sohl-Dickstein, Jascha},
 booktitle = {Advances in Neural Information Processing Systems},
 editor = {C. Cortes and N. Lawrence and D. Lee and M. Sugiyama and R. Garnett},
 pages = {},
 publisher = {Curran Associates, Inc.},
 title = {Deep Knowledge Tracing},
 url = {https://proceedings.neurips.cc/paper_files/paper/2015/file/bac9162b47c56fc8a4d2a519803d51b3-Paper.pdf},
 volume = {28},
 year = {2015}
}

@inproceedings{zhang2017dkvmn,
author = {Zhang, Jiani and Shi, Xingjian and King, Irwin and Yeung, Dit-Yan},
title = {Dynamic Key-Value Memory Networks for Knowledge Tracing},
year = {2017},
isbn = {9781450349130},
publisher = {International World Wide Web Conferences Steering Committee},
address = {Republic and Canton of Geneva, CHE},
url = {https://doi.org/10.1145/3038912.3052580},
doi = {10.1145/3038912.3052580},

booktitle = {Proceedings of the 26th International Conference on World Wide Web},
pages = {765--774},
numpages = {10},

location = {Perth, Australia},
series = {WWW '17}
}

@article{pandey2019sakt,
  author       = {Shalini Pandey and
                  George Karypis},
  title        = {A Self-Attentive model for Knowledge Tracing},
  journal      = {CoRR},
  volume       = {abs/1907.06837},
  year         = {2019},
  url          = {http://arxiv.org/abs/1907.06837},
  eprinttype   = {arXiv},
  eprint       = {1907.06837},

}

@inproceedings{ghosh2020akt,
author = {Ghosh, Aritra and Heffernan, Neil and Lan, Andrew S.},
title = {Context-Aware Attentive Knowledge Tracing},
year = {2020},
isbn = {9781450379984},
publisher = {Association for Computing Machinery},
address = {New York, NY, USA},
url = {https://doi.org/10.1145/3394486.3403282},
doi = {10.1145/3394486.3403282},

booktitle = {Proceedings of the 26th ACM SIGKDD International Conference on Knowledge Discovery \& Data Mining},
pages = {2330--2339},
numpages = {10},

location = {Virtual Event, CA, USA},
series = {KDD '20}
}

@misc{liu2022simplekt,
      title={simpleKT: A Simple But Tough-to-Beat Baseline for Knowledge Tracing}, 
      author={Zitao Liu and Qiongqiong Liu and Jiahao Chen and Shuyan Huang and Weiqi Luo},
      year={2023},
      eprint={2302.06881},
      archivePrefix={arXiv},
      primaryClass={cs.LG},
      url={https://arxiv.org/abs/2302.06881}, 
}

@misc{yeung2019deepirt,
      title={Deep-IRT: Make Deep Learning Based Knowledge Tracing Explainable Using Item Response Theory}, 
      author={Chun-Kit Yeung},
      year={2019},
      eprint={1904.11738},
      archivePrefix={arXiv},
      primaryClass={cs.LG},
      url={https://arxiv.org/abs/1904.11738}, 
}

@article{heffernan2014assistments,
	title = {The {ASSISTments} {Ecosystem}: {Building} a {Platform} that {Brings} {Scientists} and {Teachers} {Together} for {Minimally} {Invasive} {Research} on {Human} {Learning} and {Teaching}},
	volume = {24},
	issn = {1560-4306},

url = {https://doi.org/10.1007/s40593-014-0024-x},
	doi = {10.1007/s40593-014-0024-x},

language = {en},
	number = {4},

journal = {International Journal of Artificial Intelligence in Education},
	author = {Heffernan, Neil T. and Heffernan, Cristina Lindquist},
	month = dec,
	year = {2014},

pages = {470--497},
}

@inproceedings{choi2019ednet,
	address = {Cham},
	title = {{EdNet}: {A} {Large}-{Scale} {Hierarchical} {Dataset} in {Education}},
	isbn = {9783030522407},

doi = {10.1007/978-3-030-52240-7_13},

language = {en},
	booktitle = {Artificial {Intelligence} in {Education}},
	publisher = {Springer International Publishing},
	author = {Choi, Youngduck and Lee, Youngnam and Shin, Dongmin and Cho, Junghyun and Park, Seoyon and Lee, Seewoo and Baek, Jineon and Bae, Chan and Kim, Byungsoo and Heo, Jaewe},
	editor = {Bittencourt, Ig Ibert and Cukurova, Mutlu and Muldner, Kasia and Luckin, Rose and Millán, Eva},
	year = {2020},

pages = {69--73},
}

@misc{wang2021neuripsedu,
      title={Results and Insights from Diagnostic Questions: The NeurIPS 2020 Education Challenge}, 
      author={Zichao Wang and Angus Lamb and Evgeny Saveliev and Pashmina Cameron and Yordan Zaykov and Jose Miguel Hernandez-Lobato and Richard E. Turner and Richard G. Baraniuk and Craig Barton and Simon Peyton Jones and Simon Woodhead and Cheng Zhang},
      year={2021},
      eprint={2104.04034},
      archivePrefix={arXiv},
      primaryClass={cs.CY},
      url={https://arxiv.org/abs/2104.04034}, 
}

@book{cover2006elements,
	edition = {2},
	title = {Elements of {Information} {Theory}},

isbn = {9780471241959 9780471748823},
	url = {https://onlinelibrary.wiley.com/doi/book/10.1002/047174882X},
	language = {en},

publisher = {Wiley},
	author = {Cover, Thomas M. and Thomas, Joy A.},
	month = sep,
	year = {2006},
	doi = {10.1002/047174882X},
}

@article{willems1995ctw,
author = {Willems, F. M.J. and Shtarkov, Y. M. and Tjalkens, T. J.},
title = {The context-tree weighting method: basic properties},
year = {1995},
issue_date = {May 1995},
publisher = {IEEE Press},
volume = {41},
number = {3},
issn = {0018-9448},
url = {https://doi.org/10.1109/18.382012},
doi = {10.1109/18.382012},

journal = {IEEE Trans. Inf. Theor.},
month = may,
pages = {653--664},
numpages = {12}
}

@article{brier1950verification,
	title = {{VERIFICATION} {OF} {FORECASTS} {EXPRESSED} {IN} {TERMS} {OF} {PROBABILITY}},
	volume = {78},
	issn = {1520-0493, 0027-0644},
	url = {https://journals.ametsoc.org/view/journals/mwre/78/1/1520-0493_1950_078_0001_vofeit_2_0_co_2.xml},
	doi = {10.1175/1520-0493(1950)078<0001:VOFEIT>2.0.CO;2},

language = {EN},
	number = {1},

journal = {Monthly Weather Review},
	author = {Brier, Glenn W.},
	month = jan,
	year = {1950},
	pages = {1--3},
}

@article{gneiting2007strictly,
author = {Tilmann Gneiting and Adrian E Raftery},
title = {Strictly Proper Scoring Rules, Prediction, and Estimation},
journal = {Journal of the American Statistical Association},
volume = {102},
number = {477},
pages = {359--378},
year = {2007},
publisher = {Taylor \& Francis},
doi = {10.1198/016214506000001437},


URL = { 
    
        https://doi.org/10.1198/016214506000001437
    
    

},
eprint = { 
    
        https://doi.org/10.1198/016214506000001437
    
    

}

}

@inproceedings{liu2022pykt,
author = {Liu, Zitao and Liu, Qiongqiong and Chen, Jiahao and Huang, Shuyan and Tang, Jiliang and Luo, Weiqi},
title = {{pyKT}: a python library to benchmark deep learning based knowledge tracing models},
year = {2022},
isbn = {9781713871088},
publisher = {Curran Associates Inc.},
address = {Red Hook, NY, USA},

booktitle = {Proceedings of the 36th International Conference on Neural Information Processing Systems},
articleno = {1347},
numpages = {14},
location = {New Orleans, LA, USA},
series = {NIPS '22},
  doi = {10.52202/068431-1347},
  url = {https://proceedings.neurips.cc/paper_files/paper/2022/hash/75ca2b23d9794f02a92449af65a57556-Abstract-Datasets_and_Benchmarks.html}
}

@inproceedings{dktplus2018,
author = {Yeung, Chun-Kit and Yeung, Dit-Yan},
title = {Addressing two problems in deep knowledge tracing via prediction-consistent regularization},
year = {2018},
isbn = {9781450358866},
publisher = {Association for Computing Machinery},
address = {New York, NY, USA},
url = {https://doi.org/10.1145/3231644.3231647},
doi = {10.1145/3231644.3231647},

booktitle = {Proceedings of the Fifth Annual ACM Conference on Learning at Scale},
articleno = {5},
numpages = {10},

location = {London, United Kingdom},
series = {L@S '18}
}

@inproceedings{dktforget2019,
author = {Nagatani, Koki and Zhang, Qian and Sato, Masahiro and Chen, Yan-Ying and Chen, Francine and Ohkuma, Tomoko},
title = {Augmenting Knowledge Tracing by Considering Forgetting Behavior},
year = {2019},
isbn = {9781450366748},
publisher = {Association for Computing Machinery},
address = {New York, NY, USA},
url = {https://doi.org/10.1145/3308558.3313565},
doi = {10.1145/3308558.3313565},

booktitle = {The World Wide Web Conference},
pages = {3101--3107},
numpages = {7},

location = {San Francisco, CA, USA},
series = {WWW '19}
}

@inproceedings{kqn2019,
author = {Lee, Jinseok and Yeung, Dit-Yan},
title = {Knowledge Query Network for Knowledge Tracing: How Knowledge Interacts with Skills},
year = {2019},
isbn = {9781450362566},
publisher = {Association for Computing Machinery},
address = {New York, NY, USA},
url = {https://doi.org/10.1145/3303772.3303786},
doi = {10.1145/3303772.3303786},

booktitle = {Proceedings of the 9th International Conference on Learning Analytics \& Knowledge},
pages = {491--500},
numpages = {10},

location = {Tempe, AZ, USA},
series = {LAK19}
}

@inproceedings{saint2020,
author = {Choi, Youngduck and Lee, Youngnam and Cho, Junghyun and Baek, Jineon and Kim, Byungsoo and Cha, Yeongmin and Shin, Dongmin and Bae, Chan and Heo, Jaewe},
title = {Towards an Appropriate Query, Key, and Value Computation for Knowledge Tracing},
year = {2020},
isbn = {9781450379519},
publisher = {Association for Computing Machinery},
address = {New York, NY, USA},
url = {https://doi.org/10.1145/3386527.3405945},
doi = {10.1145/3386527.3405945},

booktitle = {Proceedings of the Seventh ACM Conference on Learning @ Scale},
pages = {341--344},
numpages = {4},

location = {Virtual Event, USA},
series = {L@S '20}
}

@inproceedings{atkt2021,
author = {Guo, Xiaopeng and Huang, Zhijie and Gao, Jie and Shang, Mingyu and Shu, Maojing and Sun, Jun},
title = {Enhancing Knowledge Tracing via Adversarial Training},
year = {2021},
isbn = {9781450386517},
publisher = {Association for Computing Machinery},
address = {New York, NY, USA},
url = {https://doi.org/10.1145/3474085.3475554},
doi = {10.1145/3474085.3475554},

booktitle = {Proceedings of the 29th ACM International Conference on Multimedia},
pages = {367--375},
numpages = {9},

location = {Virtual Event, China},
series = {MM '21}
}

@inproceedings{hawkes2021,
author = {Wang, Chenyang and Ma, Weizhi and Zhang, Min and Lv, Chuancheng and Wan, Fengyuan and Lin, Huijie and Tang, Taoran and Liu, Yiqun and Ma, Shaoping},
title = {Temporal Cross-Effects in Knowledge Tracing},
year = {2021},
isbn = {9781450382977},
publisher = {Association for Computing Machinery},
address = {New York, NY, USA},
url = {https://doi.org/10.1145/3437963.3441802},
doi = {10.1145/3437963.3441802},

booktitle = {Proceedings of the 14th ACM International Conference on Web Search and Data Mining},
pages = {517--525},
numpages = {9},

location = {Virtual Event, Israel},
series = {WSDM '21}
}

@inproceedings{dimkt2022,
author = {Shen, Shuanghong and Huang, Zhenya and Liu, Qi and Su, Yu and Wang, Shijin and Chen, Enhong},
title = {Assessing Student's Dynamic Knowledge State by Exploring the Question Difficulty Effect},
year = {2022},
isbn = {9781450387323},
publisher = {Association for Computing Machinery},
address = {New York, NY, USA},
url = {https://doi.org/10.1145/3477495.3531939},
doi = {10.1145/3477495.3531939},

booktitle = {Proceedings of the 45th International ACM SIGIR Conference on Research and Development in Information Retrieval},
pages = {427--437},
numpages = {11},

location = {Madrid, Spain},
series = {SIGIR '22}
}

@inproceedings{atdkt2023,
author = {Liu, Zitao and Liu, Qiongqiong and Chen, Jiahao and Huang, Shuyan and Gao, Boyu and Luo, Weiqi and Weng, Jian},
title = {Enhancing Deep Knowledge Tracing with Auxiliary Tasks},
year = {2023},
isbn = {9781450394161},
publisher = {Association for Computing Machinery},
address = {New York, NY, USA},
url = {https://doi.org/10.1145/3543507.3583866},
doi = {10.1145/3543507.3583866},

booktitle = {Proceedings of the ACM Web Conference 2023},
pages = {4178--4187},
numpages = {10},

location = {Austin, TX, USA},
series = {WWW '23}
}

@inproceedings{sparsekt2024,
author = {Huang, Shuyan and Liu, Zitao and Zhao, Xiangyu and Luo, Weiqi and Weng, Jian},
title = {Towards Robust Knowledge Tracing Models via k-Sparse Attention},
year = {2023},
isbn = {9781450394086},
publisher = {Association for Computing Machinery},
address = {New York, NY, USA},
url = {https://doi.org/10.1145/3539618.3592073},
doi = {10.1145/3539618.3592073},

booktitle = {Proceedings of the 46th International ACM SIGIR Conference on Research and Development in Information Retrieval},
pages = {2441--2445},
numpages = {5},

location = {Taipei, Taiwan},
series = {SIGIR '23}
}

@inproceedings{folibikt2023, series={CIKM '23},
   title={Forgetting-aware Linear Bias for Attentive Knowledge Tracing},
   url={http://dx.doi.org/10.1145/3583780.3615191},
   DOI={10.1145/3583780.3615191},
   booktitle={Proceedings of the 32nd ACM International Conference on Information and Knowledge Management},
   publisher={ACM},
   author={Im, Yoonjin and Choi, Eunseong and Kook, Heejin and Lee, Jongwuk},
   year={2023},
   month=Oct, pages={3958--3962},

}

@inproceedings{stablekt2024,
  title     = {Enhancing Length Generalization for Attention Based Knowledge Tracing Models with Linear Biases},
  author    = {Li, Xueyi and Bai, Youheng and Guo, Teng and Liu, Zitao and Huang, Yaying and Zhao, Xiangyu and Xia, Feng and Luo, Weiqi and Weng, Jian},
  booktitle = {Proceedings of the Thirty-Third International Joint Conference on
               Artificial Intelligence, {IJCAI-24}},
  publisher = {International Joint Conferences on Artificial Intelligence Organization},
  editor    = {Kate Larson},
  pages     = {5918--5926},
  year      = {2024},
  month     = {8},
  note      = {Main Track},
  doi       = {10.24963/ijcai.2024/654},
  url       = {https://doi.org/10.24963/ijcai.2024/654},
}

@incollection{extrakt2024,
	title = {Extending {Context} {Window} of {Attention} {Based} {Knowledge} {Tracing} {Models} via {Length} {Extrapolation}},

isbn = {9781643685489},
	url = {https://ebooks.iospress.nl/doi/10.3233/FAIA240651},

booktitle = {Frontiers in {Artificial} {Intelligence} and {Applications}},
	publisher = {IOS Press},
	author = {Li, Xueyi and Bai, Youheng and Guo, Teng and Zheng, Ying and Hou, Mingliang and Zhan, Bojun and Huang, Yaying and Liu, Zitao and Gao, Boyu and Luo, Weiqi},
	editor = {Endriss, Ulle and Melo, Francisco S. and Bach, Kerstin and Bugarín-Diz, Alberto and Alonso-Moral, José M. and Barro, Senén and Heintz, Fredrik},
	month = oct,
	year = {2024},
	doi = {10.3233/FAIA240651},
}


\appendix
\numberwithin{equation}{section}
\setcounter{fact}{0}
\renewcommand{\thefact}{\thesection.\arabic{fact}}

\section{Extended Theoretical Framework and Proofs}
\label{app:theory}

\subsection{Binary entropy properties and the LIU proxy}
For a binary correctness probability $p\in[0,1]$, the entropy in bits is
\begin{equation}
\label{eq:app-binary-entropy}
H_2(p)=-p\log_2 p-(1-p)\log_2(1-p).
\end{equation}
Its first two derivatives on $(0,1)$ are
\begin{equation}
\label{eq:app-binary-entropy-derivatives}
\frac{dH_2}{dp}(p)=\log_2\frac{1-p}{p},
\qquad
\frac{d^2H_2}{dp^2}(p)=-\frac{1}{(\ln 2)\,p(1-p)}<0.
\end{equation}
Equation~\eqref{eq:app-binary-entropy-derivatives} shows three properties that matter for the appendix. First, $H_2(p)$ is symmetric around $p=0.5$. Second, it is strictly increasing on $[0,0.5]$ and strictly decreasing on $[0.5,1]$. Third, it is strictly concave. Entropy measures the uncertainty of a predictive distribution; it is not a loss for evaluating whether a confident prediction is correct. As a result, $H_2(p)$ is small when the next response is causally predictable and maximal at one bit when the best available causal probability is near chance.

\begin{fact}[Local entropy-risk decomposition]
\label{fact:app-entropy-risk}
For any $\mathcal F_t$-measurable probabilistic predictor $\hat P_\theta$,
\begin{equation}
\label{eq:app-entropy-risk}
\mathbb E\big[-\log \hat P_\theta(R_{t+1}\mid\mathcal F_t)\big]
=H(R_{t+1}\mid\mathcal F_t)+\mathbb E\big[\mathrm{KL}(P_t\|\hat P_\theta)\big].
\end{equation}
\end{fact}

\begin{proof}
Let $P_t(r)=P(R_{t+1}=r\mid\mathcal F_t)$ for $r\in\{0,1\}$. Then
\begin{equation}
\label{eq:app-entropy-risk-step1}
\mathbb E\big[-\log \hat P_\theta(R_{t+1}\mid\mathcal F_t)\big]
=\mathbb E\Big[\sum_{r\in\{0,1\}}P_t(r)\big(-\log \hat P_\theta(r\mid\mathcal F_t)\big)\Big].
\end{equation}
Adding and subtracting $\log P_t(r)$ inside the sum yields
\begin{equation}
\label{eq:app-entropy-risk-step2}
\mathbb E\Big[\sum_r P_t(r)(-\log P_t(r))\Big]
+\mathbb E\Big[\sum_r P_t(r)\log\frac{P_t(r)}{\hat P_\theta(r\mid\mathcal F_t)}\Big].
\end{equation}
The first term in~\eqref{eq:app-entropy-risk-step2} is $H(R_{t+1}\mid\mathcal F_t)$ and the second term is $\mathbb E[\mathrm{KL}(P_t\|\hat P_\theta)]$, which proves~\eqref{eq:app-entropy-risk}.
\end{proof}

\begin{fact}[Binary Bayes error lower bound]
\label{fact:app-bayes-lower-bound}
Let $P_e^*(t)$ be the Bayes $0$-$1$ error for predicting $R_{t+1}$ from $\mathcal F_t$. Then
\begin{equation}
\label{eq:app-bayes-lower-bound}
P_e^*(t)\ge h_2^{-1}\big(H(R_{t+1}\mid\mathcal F_t)\big),
\end{equation}
where $h_2(p)=-p\log_2 p-(1-p)\log_2(1-p)$ and $h_2^{-1}$ denotes its inverse on $[0,1/2]$.
\end{fact}

\begin{proof}
For binary targets, the conditional form of Fano's inequality gives
\begin{equation}
\label{eq:app-fano}
H(R_{t+1}\mid\mathcal F_t)\le h_2(P_e^*(t)).
\end{equation}
Because $P_e^*(t)\in[0,1/2]$ and $h_2$ is strictly increasing on that interval by~\eqref{eq:app-binary-entropy-derivatives}, applying $h_2^{-1}$ to~\eqref{eq:app-fano} yields~\eqref{eq:app-bayes-lower-bound}.
\end{proof}

These facts explain why the appendix uses
\begin{equation}
\label{eq:app-uctw}
\hat u_{\mathrm{ctw},t}=H_2(\hat p_{\mathrm{ctw},t})
\end{equation}
as the operational LIU proxy. The quantity in~\eqref{eq:app-uctw} depends only on the one-step-ahead Bernoulli law implied by the causal anchor, it is invariant to threshold choices, and it is directly aligned with the cross-entropy decomposition in~\eqref{eq:app-entropy-risk}.

\subsection{CTW algorithm details and cold-start behavior}
The real-data implementation conditions binary response predictions on past item-response tokens and a current-item query token, as described in Section~3.2. Item identity is therefore part of its context. For clarity, the equations and Algorithm~\ref{alg:ctw-appendix} below review classical binary-source CTW on $R=(r_1,\dots,r_T)$; they are an explanatory special case, not a specification of the item-token context builder used for the real benchmarks. Exact reconstruction of that builder requires the implementation and its input artifacts. Let $D$ be the maximum context depth and let $\lambda$ denote the empty context. For each binary context $s$ with $|s|\le D$, let $r^s$ be the subsequence of responses whose preceding suffix equals $s$. The Krichevsky-Trofimov block probability at node $s$ is
\begin{equation}
\label{eq:app-kt-block}
P_e^s(r^s)=\prod_{j=1}^{|r^s|}\frac{n_{s,j-1}^{\,r_j^s}+\frac{1}{2}}{n_{s,j-1}^{\,0}+n_{s,j-1}^{\,1}+1},
\end{equation}
where $n_{s,j-1}^{\,a}$ counts how many times symbol $a\in\{0,1\}$ has appeared in context $s$ before the $j$-th symbol of $r^s$.

The CTW mixture recursively combines the local KT model with the two child contexts:
\begin{equation}
\label{eq:app-ctw-recursion}
P_w^s(r^s)=
\begin{cases}
P_e^s(r^s), & |s|=D,\\[0.5ex]
\frac{1}{2}P_e^s(r^s)+\frac{1}{2}P_w^{0s}(r^{0s})P_w^{1s}(r^{1s}), & |s|<D.
\end{cases}
\end{equation}
The one-step predictive probability at the root is then
\begin{equation}
\label{eq:app-ctw-predictive}
\hat p_{\mathrm{ctw},t}
=\Pr_{\mathrm{CTW}}(R_{t+1}=1\mid r_{1:t})
=\frac{P_w^\lambda(r_{1:t}1)}{P_w^\lambda(r_{1:t}0)+P_w^\lambda(r_{1:t}1)}.
\end{equation}

\begin{algorithm}[t]
\caption{Classical binary-source CTW: an illustrative special case}
\label{alg:ctw-appendix}
\begin{algorithmic}[1]
\STATE {\bfseries Input:} Binary response stream $R=(r_1,\dots,r_T)$, maximum depth $D$
\STATE Initialize a depth-$D$ context tree with KT priors at every node
\FOR{$t=0$ {\bfseries to} $T-1$}
    \STATE Traverse the suffix path induced by the observed prefix $r_{1:t}$
    \STATE Compute the root mixture probability $\hat p_{\mathrm{ctw},t}$ by the CTW recursion in~\eqref{eq:app-ctw-recursion} and the predictive ratio in~\eqref{eq:app-ctw-predictive}
    \STATE Convert the anchor probability into uncertainty: $\hat u_{\mathrm{ctw},t}\leftarrow H_2(\hat p_{\mathrm{ctw},t})$
    \STATE Observe $r_{t+1}$ and update the node statistics on the same suffix path
\ENDFOR
\STATE {\bfseries Output:} Anchor probabilities $\{\hat p_{\mathrm{ctw},t}\}_{t=0}^{T-1}$ and uncertainty profile $\{\hat u_{\mathrm{ctw},t}\}_{t=0}^{T-1}$
\end{algorithmic}
\end{algorithm}

\paragraph{Cold-start behavior.}
At the beginning of a student sequence, most contexts are unseen. In that regime, the KT prior in~\eqref{eq:app-kt-block} keeps $\hat p_{\mathrm{ctw},t}$ close to $0.5$, so $\hat u_{\mathrm{ctw},t}$ in~\eqref{eq:app-uctw} stays close to one bit rather than becoming spuriously overconfident. As more evidence accumulates, longer contexts contribute through~\eqref{eq:app-ctw-recursion}. For stationary finite-depth tree sources, the classical CTW per-symbol redundancy decays as $O(\log t/t)$ for a fixed source class \citep{willems1995ctw}, so the cold-start penalty vanishes with sequence length. In this paper we use that result only as a property of the reference source model: student learning itself is not assumed stationary, but the reference guarantee clarifies why CTW is a stable causal anchor.

\subsection{Mutual-information decomposition into RU and IG}
For an item $q$, define the item-marginal entropy and the oracle conditional entropy by
\begin{equation}
\label{eq:app-item-entropies}
H_q^{\mathrm{marg}}:=H(R_{t+1}\mid q_{t+1}=q),
\qquad
H_q^{\mathrm{cond}}:=H(R_{t+1}\mid X_t,q_{t+1}=q).
\end{equation}
The theoretical remaining uncertainty and information gain are
\begin{equation}
\label{eq:app-ru-ig-theoretical}
\mathrm{RU}^\star(q):=H_q^{\mathrm{cond}},
\qquad
\mathrm{IG}^\star(q):=H_q^{\mathrm{marg}}-H_q^{\mathrm{cond}}.
\end{equation}

\begin{fact}[RU/IG identity]
\label{fact:app-ru-ig}
For every item $q$,
\begin{equation}
\label{eq:app-ru-ig-identity}
H_q^{\mathrm{marg}}=\mathrm{RU}^\star(q)+\mathrm{IG}^\star(q),
\qquad
\mathrm{IG}^\star(q)=I(R_{t+1};X_t\mid q_{t+1}=q)\ge 0.
\end{equation}
\end{fact}

\begin{proof}
By the definition of conditional mutual information,
\begin{equation}
\label{eq:app-ru-ig-proof}
I(R_{t+1};X_t\mid q_{t+1}=q)
=H(R_{t+1}\mid q_{t+1}=q)-H(R_{t+1}\mid X_t,q_{t+1}=q).
\end{equation}
Substituting~\eqref{eq:app-item-entropies} into~\eqref{eq:app-ru-ig-proof} gives the second identity in~\eqref{eq:app-ru-ig-identity}; the first follows by rearrangement, and non-negativity is a standard property of mutual information.
\end{proof}

The paper uses CTW to obtain operational plug-in diagnostics:
\begin{equation}
\label{eq:app-ru-ig-operational}
\mathrm{RU}(q):=H_q^{\mathrm{ctw}},
\qquad
\mathrm{IG}(q):=H_q^{\mathrm{marg}}-H_q^{\mathrm{ctw}}.
\end{equation}
Because~\eqref{eq:app-ru-ig-operational} replaces the oracle conditional entropy with an estimated one, negative $\mathrm{IG}(q)$ values may appear because of finite-sample estimation or anchor misspecification; the non-negativity guarantee applies to the oracle mutual information, not to this plug-in diagnostic.

\paragraph{Venn-style interpretation.}
It is useful to read~\eqref{eq:app-ru-ig-identity} as a verbal Venn diagram. The full circle is the item-marginal entropy $H_q^{\mathrm{marg}}$. The portion explained by student history is $\mathrm{IG}^\star(q)$, while the unexplained remainder is $\mathrm{RU}^\star(q)$. High IG therefore means that historical context removes a large fraction of the original uncertainty; high RU means that substantial uncertainty remains even after conditioning on the observable history.

\paragraph{Interpretation guide.}
If RU is high and IG is low, the item is a warning case: the item-conditioned anchor leaves high uncertainty relative to item-marginal entropy. This motivates checking guessing, item ambiguity, label quality, missing variables, and anchor misspecification; it does not prove intrinsic randomness. If IG is high, the benchmark still contains exploitable contextual structure, so stronger temporal, prerequisite, or item-context modeling is warranted.


\section{Data and Preprocessing Details}
\label{app:data}

\subsection{Dataset coverage}
Table~\ref{tab:app-dataset-coverage} summarizes the dataset coverage used in the appendix. The highest-band mass refers to the CTW-entropy band $[0.8,1.0]$ used in the main benchmark-diagnosis figures.

\begin{table}[t]
\centering
\caption{Dataset coverage used by the appendix. Support-interaction counts are reported for NIPS Task 3/4 and Algebra 2005.}
\label{tab:app-dataset-coverage}
\footnotesize
\begin{tabular}{@{}L{0.15\linewidth}L{0.15\linewidth}L{0.18\linewidth}L{0.24\linewidth}L{0.14\linewidth}@{}}
\toprule
Dataset & Protocol & Support interactions & Aligned evaluation points & Highest-band mass \\
\midrule
NIPS Task 3/4 & question-level & $898{,}674$ & $1{,}354{,}735$ & $74.8\%$ \\
Algebra 2005 & question-level & $598{,}954$ & $464{,}725$ & $55.6\%$ \\
ASSIST2015 & $concept=question$ & -- & $672{,}960$ for CTW+NN/KT; $673{,}200$ for CTW before KT-side intersection & -- \\
\bottomrule
\end{tabular}
\end{table}

\subsection{Preprocessing protocol and alignment rules}
The main benchmark-diagnosis figures use question-level preprocessing for NIPS Task 3/4 and Algebra 2005. ASSIST2015 is reported as an auxiliary consistency check under the convention $concept=question$, which makes the prediction unit compatible with the question-level protocol used by the two primary benchmarks. All reported comparisons use five-fold aligned held-out points: CTW-side anchor probabilities and task-facing KT predictions are joined only when they share the same fold, student/history position, item identifier, and held-out response.

The diagnostic stage operates on binary correctness labels and does not introduce a separate appendix-specific prediction task beyond the benchmark sequence protocol used by the main paper.

\subsection{Knowledge-component mapping and Q-matrix status}
The CTW anchor itself does not require a Q-matrix. It uses historical item-response tokens and the next queried item's token as predictor inputs, in addition to using item identifiers for prediction alignment and item-level diagnostics. Prediction precedes the current-response update. The appendix does not introduce an additional inferred KC mapping layer.


\section{Implementation and Reproducibility}
\label{app:repro}

\subsection{Implementation settings}
Table~\ref{tab:app-implementation-settings} collects the implementation settings used by the appendix.

\begin{table}[t]
\centering
\caption{Implementation settings used by the appendix.}
\label{tab:app-implementation-settings}
\footnotesize
\begin{tabular}{@{}L{0.32\linewidth}L{0.58\linewidth}@{}}
\toprule
Component & Setting \\
\midrule
CTW anchor & Item-response token stream (past \texttt{I:item|R:resp} tokens plus the current \texttt{Q:item} query token), maximum context depth $D=6$, KT priors, causal one-step-ahead prediction. The token budget includes one query token and at most $D-1=5$ historical interaction tokens (effective response depth $5$). \\
Uncertainty bands & Five fixed CTW-entropy bands: $[0.0,0.2)$, $[0.2,0.4)$, $[0.4,0.6)$, $[0.6,0.8)$, $[0.8,1.0]$. \\
CTW+NN probe & Fused logit correction $\sigma(\operatorname{logit}(\hat p_{\mathrm{ctw},t})+f_\phi(X_{\le t},Q_{t+1}))$ used only as an anchor-side neural probe. \\
Synthetic sanity check & $64$ trials, sequence length $400$, one-step warmup, CTW depth $6$. \\
Item diagnostics & Minimum support $50$ per item; diagnostic guide lines at $RU=0.8$ and $IG=0.1$. \\
\bottomrule
\end{tabular}
\end{table}

\subsection{Training and evaluation details}
All global metrics and local diagnostic summaries are reported as mean $\pm$ standard deviation over five folds. The CTW anchor and the derived bucketization are always computed out-of-fold. The shared aligned subset contains $1{,}354{,}735$ points on NIPS Task 3/4, $464{,}725$ points on Algebra 2005, and $672{,}960$ points on ASSIST2015 under the convention $concept=question$.

The archived run exports used for the retained KT comparison table preserve fold-specific hyperparameter fields for each task-facing baseline. Table~\ref{tab:app-recovered-hparams} summarizes representative recovered settings for the main baselines discussed in the paper. Each entry reports the set of values that appears across the five retained folds for that dataset, so multi-valued cells indicate genuine fold-level variation rather than post hoc averaging.

\begin{table*}[t]
\centering
\caption{Recovered fold-specific hyperparameter summaries for representative KT baselines in the archived run exports. Each cell reports the values appearing across the five retained folds for that dataset. Abbreviations: $d$ for hidden width, $d_{ff}$ for feed-forward width, heads for attention heads, $B$ for transformer blocks, and mem for memory size.}
\label{tab:app-recovered-hparams}
\footnotesize
\setlength{\tabcolsep}{3pt}
\begin{tabular}{@{}L{0.10\linewidth}L{0.28\linewidth}L{0.28\linewidth}L{0.28\linewidth}@{}}
\toprule
Model & NIPS Task 3/4 & Algebra 2005 & ASSIST2015 \\
\midrule
DKT & emb $256$, drop $0.3/0.5$, lr $10^{-3}$ & emb $256$, drop $0.1/0.3/0.5$, lr $10^{-3}$ & emb $256$, drop $0.05/0.3/0.5$, lr $10^{-3}$ \\
AKT & $d=64$, $d_{ff}=64/256$, heads $4/8$, $B=2/4$, drop $0.1/0.3$, lr $10^{-3}$ & $d=64/256$, $d_{ff}=256$, heads $4$, $B=4$, drop $0.05/0.1$, lr $10^{-4}$ & $d=64/256$, $d_{ff}=64/256$, heads $4/8$, $B=4$, drop $0.05/0.1$, lr $10^{-4}/10^{-5}$ \\
simpleKT & $d=64$, $d_{ff}=256$, layers $2$, $B=4$, drop $0.3$, lr $10^{-3}$ & $d=256$, $d_{ff}=64/256$, layers $2$, $B=4$, drop $0.05/0.1$, lr $10^{-4}/10^{-5}$ & $d=256$, $d_{ff}=256$, layers $2$, $B=4$, drop $0.05/0.1$, lr $10^{-4}/10^{-5}$ \\
stableKT & $d=64$, $d_{ff}=64/256$, layers $2$, heads $4$, drop $0.3$, lr $10^{-3}$, $r=0.1/0.7/0.9/1.0$, $\gamma=0.7/0.9/1.0/5.0$ & $d=256$, $d_{ff}=64/256$, layers $2$, heads $4$, drop $0.05/0.1$, lr $10^{-4}/10^{-5}$, $r=0.5/0.9/5.0$, $\gamma=0.1/0.3/0.5/5.0$ & $d=256$, $d_{ff}=64/256$, layers $2$, heads $4$, drop $0.05/0.1$, lr $10^{-5}$, $r=0.5/1.0/2.0$, $\gamma=0.1/0.5/0.7$ \\
Deep-IRT & dim$_s=256$, mem $32/64$, drop $0.05/0.3$, lr $10^{-3}$ & dim$_s=256$, mem $32$, drop $0.05/0.1$, lr $10^{-3}$ & dim$_s=64$, mem $32$, drop $0.05/0.1$, lr $10^{-3}$ \\
ATKT & skill $64$, ans $64/256$, hidden $256$, attn $64/256$, drop $0.05/0.1$, lr $10^{-4}/10^{-3}$ & skill $64$, ans $64$, hidden $256$, attn $256$, drop $0.05/0.1/0.5$, lr $10^{-3}$ & skill $64$, ans $64/256$, hidden $256$, attn $64/256$, drop $0.05$, lr $10^{-4}/10^{-3}$ \\
\bottomrule
\end{tabular}
\end{table*}

The same archive also preserves explicit training-length metadata for DIMKT on the two main benchmark datasets: batch size $64$, num\_steps $199$, learning rate $10^{-3}$, and weight decay $10^{-4}$ on both NIPS Task 3/4 and Algebra 2005, with dropout varying over $\{0.05,0.1,0.3\}$. We did not find optimizer, early-stopping, or wall-clock fields in the retained run exports, so those quantities are not supplied for the original baseline branch. The separately reported CTW+NN-SKT branch has validation-NLL selection and early-stopping settings described in Appendix~\ref{app:robustness}.

\subsection{Compute-cost discussion}
The CTW anchor is intentionally lightweight. For a fixed depth $D$, each prediction/update step touches only the suffix path of length at most $D$, so the per-step symbolic overhead is linear in $D$ and small in practice. In our experiments we use $D=6$, which is sufficient to separate the i.i.d. and persistent synthetic regimes while keeping the causal anchor inexpensive. This is one of the main practical advantages of the framework: the uncertainty analysis adds a modest symbolic pass on top of predictions that are already being generated for KT evaluation, whereas the expensive part of the pipeline remains the training of the task-facing neural KT models.


\section{Additional Experimental Results}
\label{app:additional}

\subsection{Cross-benchmark diagnostic differences}
Table~\ref{tab:app-cross-benchmark} summarizes the strongest cross-benchmark contrasts reported in the paper. The main point is not only that the high-entropy region exists, but that it carries different mass and different residual headroom across datasets.

\begin{table}[t]
\centering
\caption{Cross-benchmark diagnostic contrasts. Positive $\Delta$ACC means CTW+NN improves over CTW; negative $\Delta$Brier means the proper score improves.}
\label{tab:app-cross-benchmark}
\footnotesize
\begin{tabular}{@{}L{0.16\linewidth}L{0.15\linewidth}L{0.15\linewidth}L{0.18\linewidth}L{0.22\linewidth}@{}}
\toprule
Dataset & Highest CTW band mass & $\Delta$ACC in highest CTW band & $\Delta$Brier in highest CTW band & Refined-coordinate note \\
\midrule
NIPS Task 3/4 & $74.8\%$ & $+0.0886$ & $-0.0378$ & Under refined repartitioning, the peak CTW+NN gain shifts to the upper-middle bucket at $+0.0868$. \\
Algebra 2005 & $55.6\%$ & $+0.0283$ & $-0.0681$ & Under refined repartitioning, the largest CTW+NN gain appears in the fourth bucket at $+0.0335$, while gains remain positive in the highest bucket. \\
\bottomrule
\end{tabular}
\end{table}

The table reinforces the same diagnosis from two angles. On NIPS Task 3/4, the highest-entropy region dominates the benchmark mass and still admits the largest accuracy gain over CTW, showing residual predictive structure relative to this anchor in its difficult regime. On Algebra 2005, the hard region is smaller but still substantial, and the largest proper-score improvement also appears there. Thus RU/IG and band-level gains should be read together: high entropy can still hide exploitable structure, but the gain pattern differs by dataset.


\clearpage
\section{Robustness to Anchor Information, Band Resolution, and Cold Start}
\label{app:robustness}

\subsection{Analysis branches and reporting conventions}
The original reported Table 2 branch is identified as a single SimpleKT residual export, \texttt{simplekt\_residual\_ctw\_local\_bayes}, with the aligned subsets specified in the main text. The additional analysis in this appendix uses a single SimpleKT residual probe, CTW+NN-SKT, and its common evaluation cohort. These are separate analysis branches: differences between their point estimates should not be interpreted as reruns of an identical predictor or cohort. Within each additional comparison, evaluated prediction rows are held fixed across the reference and probe. Equal-frequency CTW quintiles contain approximately $271{,}804$ NIPS predictions or $93{,}171$ Algebra predictions per quintile. These approximate counts are not a replacement for the original aligned-subset totals.

All deltas are model minus reference on matched predictions. Positive $\Delta$ACC and negative $\Delta$Brier or $\Delta$Log-loss indicate improvement. Additional log-loss values use natural logarithms and are reported in nats, whereas uncertainty coordinates are in bits. Dividing log-loss by $\ln 2$ converts it to bits without changing comparisons. The theoretical decomposition applies to log-loss; Brier, accuracy, and calibration are complementary empirical checks.

\subsection{Independent item-aware causal references}
ItemMean is an out-of-fold empirical-Bayes reference. With $c_q$ correct responses and $n_q$ observations for item $q$ in the four training folds, training-fold global correctness $\bar p$, and prior strength $\alpha$, it predicts
\begin{equation}
\hat p_{\mathrm{item}}(q)=\frac{c_q+\alpha\bar p}{n_q+\alpha}.
\end{equation}
The default $\alpha=20$ supplies 20 global-rate pseudo-observations; unseen items fall back to $\bar p$. Held-out labels are excluded from these item counts.

ItemRasch adds a one-dimensional causal online student state to the item logit. It is a 1PL-style online reference, not a jointly fitted IRT model. For each sequence sorted by prediction position, initialize $\theta_0=0$ and precision $\lambda_0=1$. Before observing response $r_t$, compute
\begin{equation}
b_t=\operatorname{logit}(\hat p_{\mathrm{item}}(q_t)),\qquad
p_t=\sigma(b_t+\theta_{t-1}).
\end{equation}
After observing the response, update
\begin{equation}
\lambda_t=\lambda_{t-1}+\max\{p_t(1-p_t),10^{-6}\},\qquad
\theta_t=\operatorname{clip}\!\left(\theta_{t-1}+\frac{r_t-p_t}{\lambda_t},-4,4\right).
\end{equation}
The item-prior probability is clipped to $[10^{-12},1-10^{-12}]$ before taking its logit. The argument passed to the sigmoid is clipped to $[-30,30]$; there is no additional $10^{-12}$ clamp after the sigmoid, whose extreme outputs are approximately $9.36\times10^{-14}$ and $1-9.36\times10^{-14}$. The reference has no neural parameters, offline student-ability fitting, optimizer, batch size, epochs, or validation-based selection. Its effective update step is $1/\lambda_t$.

\begin{table}[htbp]
\centering
\caption{Independent reference scores on the additional common cohort. Lower is better; log-loss is in nats.}
\label{tab:strong-reference}
\begin{tabular}{llcc}
\toprule
Dataset & Reference & Brier & Log-loss \\
\midrule
NIPS Task 3/4 & CTW & 0.2111 & 0.6096 \\
 & ItemRasch & 0.1923 & 0.5649 \\
Algebra 2005 & CTW & 0.1684 & 0.5016 \\
 & ItemRasch & 0.1585 & 0.4853 \\
\bottomrule
\end{tabular}
\end{table}

\subsection{Highest-band gains and uncertainty estimates}
Table~\ref{tab:robust-highband} reports five-fold mean gains with 95\% $t$ intervals. The ItemRasch rows use ItemRasch entropy to define the band and ItemRasch as the reference; they therefore test localization under an alternative item-and-ability reference, rather than adding item identity to a response-only CTW reference. AKT, simpleKT, and stableKT also improve accuracy, Brier, and log-loss in the highest ItemRasch-entropy band on both datasets.

\begin{table}[htbp]
\centering
\caption{CTW+NN-SKT gains in the highest-entropy band, $[0.8,1.0]$, on each reference's own coordinate. Entries are five-fold means [95\% $t$ intervals]; log-loss is in nats.}
\label{tab:robust-highband}
\small
\begin{tabular}{@{}llrrr@{}}
\toprule
Dataset & Reference & $\Delta$ACC & $\Delta$Brier & $\Delta$Log-loss \\
\midrule
NIPS & CTW & $+0.0905$ & $-0.0388$ & $-0.0902$ \\
 & & $[0.0898,0.0913]$ & $[-0.0389,-0.0386]$ & $[-0.0906,-0.0899]$ \\
Algebra & CTW & $+0.0285$ & $-0.0685$ & $-0.1578$ \\
 & & $[0.0252,0.0317]$ & $[-0.0701,-0.0669]$ & $[-0.1624,-0.1533]$ \\
NIPS & ItemRasch & $+0.0420$ & $-0.0175$ & $-0.0394$ \\
 & & $[0.0411,0.0429]$ & $[-0.0178,-0.0172]$ & $[-0.0400,-0.0388]$ \\
Algebra & ItemRasch & $+0.0788$ & $-0.0430$ & $-0.1009$ \\
 & & $[0.0765,0.0812]$ & $[-0.0450,-0.0411]$ & $[-0.1065,-0.0952]$ \\
\bottomrule
\end{tabular}
\end{table}

For the CTW coordinate, two-sided paired tests over folds compare the highest and lowest bands. For $\Delta$ACC, $p=1.16\times10^{-8}$ on NIPS and $p=1.71\times10^{-5}$ on Algebra; for $\Delta$Brier, $p=2.83\times10^{-6}$ and $p=2.92\times10^{-8}$; for $\Delta$Log-loss, $p=7.77\times10^{-4}$ and $p=2.89\times10^{-8}$. Under the ItemRasch coordinate, high-versus-low differences are significant for all three metrics, with $p\le4.05\times10^{-5}$ on NIPS and $p\le4.78\times10^{-6}$ on Algebra. These are fold-level descriptive comparisons rather than guarantees for new populations; the folds share training data.

Within-band expected calibration error (ECE) uses 15 fixed probability bins over $[0,1]$. In the highest CTW-entropy band, CTW+NN-SKT has ECE $0.0058$ on NIPS and $0.0264$ on Algebra. The corresponding ECE changes are $-0.0043$ [$-0.0056,-0.0030$] and $-0.1496$ [$-0.1565,-0.1428$]. ECE depends on binning and is interpreted alongside Brier and log-loss.

\subsection{Equal-frequency resolution check}
Table~\ref{tab:quintiles} resolves the large fixed-width high-entropy population with balanced CTW quintiles. On NIPS, accuracy gain increases across quintiles; on Algebra, accuracy gains concentrate in the upper two quintiles, while the largest log-loss reduction appears in the hardest quintile. Thus the evidence supports non-uniform localization rather than a universal monotonic accuracy curve.

\begin{table}[htbp]
\centering
\caption{CTW+NN-SKT minus CTW on equal-frequency CTW-entropy quintiles. Q1 has the lowest entropy and Q5 the highest; log-loss is in nats.}
\label{tab:quintiles}
\small
\begin{tabular}{llrrrrr}
\toprule
Dataset & Metric & Q1 & Q2 & Q3 & Q4 & Q5 \\
\midrule
NIPS & $\Delta$ACC & $+0.0094$ & $+0.0301$ & $+0.0637$ & $+0.0926$ & $+0.1570$ \\
 & $\Delta$Log-loss & $-0.0665$ & $-0.0705$ & $-0.0806$ & $-0.0816$ & $-0.1227$ \\
Algebra & $\Delta$ACC & $-0.0001$ & $+0.0011$ & $+0.0157$ & $+0.0433$ & $+0.0206$ \\
 & $\Delta$Log-loss & $-0.0087$ & $-0.0176$ & $-0.0404$ & $-0.1443$ & $-0.2590$ \\
\bottomrule
\end{tabular}
\end{table}

Using independent ItemRasch quintiles gives the same directional check: CTW+NN-SKT's accuracy gain over ItemRasch rises from $+0.0034$ to $+0.0853$ across NIPS quintiles and from $+0.0012$ to $+0.1036$ across Algebra quintiles. Brier and log-loss improve in every ItemRasch quintile.

\subsection{Reference-parameter sensitivity}
We vary item prior strength over $\{1,5,20,100,200\}$ and initial Rasch precision over $\{0.25,0.5,1,2,4\}$. Across these 25 configurations, CTW+NN-SKT remains favorable over ItemRasch in the highest ItemRasch-entropy band. On NIPS, $\Delta$ACC ranges from $+0.0410$ to $+0.0449$, $\Delta$Brier from $-0.0191$ to $-0.0174$, and $\Delta$Log-loss from $-0.0432$ to $-0.0391$. On Algebra, the corresponding ranges are $+0.0747$ to $+0.0990$, $-0.0553$ to $-0.0411$, and $-0.1332$ to $-0.0952$. This sensitivity check addresses two reference settings, not every possible uncertainty estimator.

\subsection{Interaction position, item support, and RU/IG regions}
The additional common cohort's highest CTW-entropy band contains $1{,}015{,}423$ NIPS predictions and $258{,}740$ Algebra predictions. Only $6.61\%$ and $2.38\%$, respectively, are in the first 20 interactions. Equivalently, $93.39\%$ and $97.62\%$ occur after position 20; median positions are 168 and 600. High ItemRasch entropy is similarly persistent, with $92.85\%$ and $97.81\%$ after position 20. The high-entropy region is therefore not dominated by initial CTW cold start.

Zero training-item support accounts for $1.22\%$ of the NIPS test intersection and $47.33\%$ of the Algebra test intersection. Within the high ItemRasch-entropy band of that slice, CTW+NN-SKT improves over ItemRasch by $+0.1056$ ACC, $-0.0353$ Brier, and $-0.0757$ log-loss on NIPS, and by $+0.1165$, $-0.0609$, and $-0.1399$ on Algebra. These are support-slice checks, not a full out-of-distribution generalization claim.

For model diagnosis, the additional NIPS high-IG region contains $250{,}315$ predictions. stableKT improves over CTW by $+0.0638$ ACC [$0.0628,0.0649$], $-0.0325$ Brier [$-0.0328,-0.0321$], and $-0.0843$ log-loss [$-0.0857,-0.0829$]. This supported example complements the original scatter plots; it does not supply inferential estimates for every plotted cell.

\subsection{Single-probe architecture and training protocol}
CTW+NN-SKT adds a SimpleKT residual to the CTW logit. The strict checkpoints use 2 or 4 SimpleKT blocks, four attention heads, and feed-forward width 256. Model selection uses validation negative log-likelihood, with at most 50 epochs and early-stopping patience 4. NIPS settings use hidden width 64 or 128, dropout 0.10 or 0.20, learning rate $5\times10^{-4}$ or $10^{-3}$, and batch size 64 or 128. Algebra settings use hidden width 128 or 256, dropout 0.05, 0.10, or 0.20, learning rate $10^{-4}$, and batch size 64 or 128.

These reported ranges document the additional branch but do not identify every selected dataset-fold configuration. Complete selected configurations, seeds, best epochs, validation NLL values, and optimizer metadata are required for exact retraining; they are not available in the provided archive. The archived aggregate results and the reported additional comparisons support the diagnostic analysis, while exact retraining remains a reproducibility limitation.

\section*{Data and code availability}
The study uses the existing educational benchmarks cited in the paper. Residual-probe training entry points are available in the supplementary material of the accepted OpenReview record at \url{https://openreview.net/forum?id=qISxwVTG3L}. Benchmark access remains subject to the original providers' terms.

\section*{Responsible use and writing assistance}
This study analyzes existing educational benchmarks and does not collect new human-subject data. The diagnostic scores should guide benchmark inspection and model evaluation; they should not be used alone to make decisions about individual students. High anchor entropy can reflect missing variables or model mismatch rather than an intrinsic limitation of a learner. Data access and reuse remain subject to each original provider's terms. LLM tools were used for writing, editing, and formatting assistance.


\end{document}